\documentclass[10pt,journal]{IEEEtran}
\IEEEoverridecommandlockouts
\usepackage{caption}
\usepackage{subcaption}
\ifCLASSINFOpdf
\else
   \usepackage[dvips]{graphicx}
\fi
\usepackage{url}
\usepackage{amsmath}
\usepackage{amsthm}
\usepackage{amssymb}
\usepackage[ruled,vlined]{algorithm2e}
\newtheorem{theorem}{Theorem}
\newtheorem{lemma}{Lemma}

\newtheorem{remark}{Remark}
\def\BibTeX{{\rm B\kern-.05em{\sc i\kern-.025em b}\kern-.08em T\kern-.1667em\lower.7ex\hbox{E}\kern-.125emX}}
\newenvironment{customdefinition}[1]{%
\par\vspace{\baselineskip}\noindent\textbf{Definition #1:}\quad\ignorespaces
}{%
\par\vspace{\baselineskip}\ignorespacesafterend
}
\usepackage{graphicx}

\DeclareMathOperator*{\argmin}{arg\,min}

\usepackage[nospace]{cite}
\usepackage{hyperref}
\usepackage{adjustbox}
\usepackage{textcomp}
\makeatletter
\newcommand{\distas}[1]{\mathbin{\overset{#1}{\kern\z@\sim}}}%
\newsavebox{\mybox}\newsavebox{\mysim}
\newcommand{\distras}[1]{%
  \savebox{\mybox}{\hbox{\kern3pt$\scriptstyle#1$\kern3pt}}%
  \savebox{\mysim}{\hbox{$\sim$}}%
  \mathbin{\overset{#1}{\kern\z@\resizebox{\wd\mybox}{\ht\mysim}{$\sim$}}}%
}
 \def\BibTeX{{\rm B\kern-.05em{\sc i\kern-.025em b}\kern-.08em
     T\kern-.1667em\lower.7ex\hbox{E}\kern-.125emX}}

 \title{Federated Smoothing Proximal Gradient for Quantile Regression with Non-Convex Penalties}

\author{Reza Mirzaeifard~\IEEEmembership{Student Member,~IEEE}, 
        Diyako Ghaderyan, 
        Stefan Werner$^{\star}$~\IEEEmembership{Fellow,~IEEE}
\thanks{Stefan Werner and Reza Mirzaeifard are with the Department
of Electronic Systems, Norwegian University of Science and Technology-NTNU, Norway,
Trondheim, 7032 Norway (e-mail: \{stefan.werner, reza.mirzaeifard\}@ntnu.no). Stefan Werner is also with the Department of Information and Communications Engineering, Aalto University, 00076, Finland. 
Diyako Ghaderyan is with the Department of Information and Communications Engineering, Aalto University, 00076, Finland (e-mail: diyako.ghaderyan@aalto.fi).
This work was partially supported by the Research Council of Norway and the Research Council of Finland (Grant 354523).}
}
\begin{document}
\maketitle
 \begin{abstract}
 Distributed sensors in the internet-of-things (IoT) generate vast amounts of sparse data. Analyzing this high-dimensional data and identifying relevant predictors pose substantial challenges, especially when data is preferred to remain on the device where it was collected for reasons such as data integrity, communication bandwidth, and privacy. This paper introduces a federated quantile regression algorithm to address these challenges. Quantile regression provides a more comprehensive view of the relationship between variables than mean regression models. However, traditional approaches face difficulties when dealing with nonconvex sparse penalties and the inherent non-smoothness of the loss function. For this purpose, we propose a federated smoothing proximal gradient (FSPG) algorithm that integrates a smoothing mechanism with the proximal gradient framework, thereby enhancing both precision and computational speed. This integration adeptly handles optimization over a network of devices, each holding local data samples, making it particularly effective in federated learning scenarios. The FSPG algorithm ensures steady progress and reliable convergence in each iteration by maintaining or reducing the value of the objective function. By leveraging nonconvex penalties, such as the minimax concave penalty (MCP) and smoothly clipped absolute deviation (SCAD), the proposed method can identify and preserve key predictors within sparse models. Comprehensive simulations validate the robust theoretical foundations of the proposed algorithm and demonstrate improved estimation precision and reliable convergence.

\end{abstract}

\IEEEpeerreviewmaketitle
\section{Introduction}
\label{sec:intro}

The internet-of-things (IoT) landscape has been transformed by the adoption of cyber-physical systems characterized by numerous distributed devices and sensors that revolutionize data collection and decision-making processes. In this decentralized environment, traditional methods relying on centralized data aggregation are impractical due to substantial computational, energy, and bandwidth demands, as well as significant privacy concerns \cite{chen2019quantile,zhou2022admm}. Federated learning (FL) offers a solution by enabling collaborative model training across edge devices, reducing privacy risks and the need for centralized data storage \cite{yin2021comprehensive,zhou2023decentralized}. 

Despite these advancements, FL struggles with outlier data, particularly from heavy-tailed distributions, which can distort learning outcomes and undermine model reliability \cite{zhao2022participant,tsouvalas2022federated}. This highlights the need for robust approaches like quantile regression \cite{koenker1987algorithm}, which analyzes predictor-response relationships across different data quantiles. Quantile regression is particularly beneficial in applications such as wind power prediction \cite{yu2020probabilistic}, uncertainty estimation in smart meter data \cite{taieb2016fore}, and load forecasting in smart grids \cite{happy2021stat}, where it manages data variability and intermittency, enhancing stability and efficiency.

Sparse regression techniques are crucial for managing complex IoT data, as they efficiently process vast, heterogeneous data from distributed sensors, providing robust solutions across various fields \cite{wu2009variable,xue2012positive}. For instance, in genetics, they help elucidate quantitative traits \cite{he2016regularized}; in bioinformatics, they refine gene selection in microarray studies \cite{algamal2018gene}; in finance, they enhance risk management models \cite{tibshirani2014adaptive}; and in ecology, they clarify relationships between environmental factors and species distribution \cite{chen2021quantile}. To further improve adaptability and efficacy within federated learning, sophisticated penalization methods like the minimax concave penalty (MCP) \cite{zhang2010nearly} and smoothly clipped absolute deviation (SCAD) \cite{fan2001variable} could be integrated. Despite their non-convex and non-smooth nature, MCP and SCAD have some advantages that make them interesting alternatives to the traditional $l_1$-penalty, e.g., they selectively shrink coefficients, effectively reducing the bias and adapting more flexibly across the sparsity spectrum of models \cite{mirzaeifard2022dynamic,mirzaeifard2022robust,mirzaeifard2022admm}.

However, despite numerous available optimization techniques for $l_1$-penalized quantile regression, including sub-gradient methods, primal-dual approaches, and the alternating direction method of multipliers (ADMM) which have established a solid framework for addressing sparsity in centralized settings \cite{belloni2011l1,koenker2005frisch,wang2017distributed,boyd2011distributed}, the optimization techniques for SCAD or MCP penalized models remains limited. Initially, techniques such as majorization-minimization (MM) and local linear approximation (LLA) have been employed to construct and solve surrogate convex functions that approximate the original non-convex problems \cite{peng2015iterative,zou2008one}. Although these methods facilitate the optimization process, they often involve secondary convergence iterations within each loop, which can result in slower overall convergence and potential precision losses. Building on these concepts, the recently proposed sub-gradient algorithm specifically addresses weakly convex functions and offers a more streamlined approach to handling non-convex penalties in quantile regression \cite{davis2019stochastic}. A recent innovation in this field is the Single-loop Iterative ADMM (SIAD) algorithm, which refines non-convex optimization strategies by employing smoothing techniques that eliminate the need for a smooth component or an implicit Lipschitz condition, thereby enabling efficient convergence \cite{mirzaeifard2022admm}. Despite the promising advancements introduced by SIAD, its application has been limited to centralized settings. Given the growing interest in distributed models for quantile regression, especially within the federated framework, there is a need for algorithms capable of operating effectively across multiple nodes without sharing actual data.

The federated learning field is rich with algorithms offering unique features such as asynchronicity \cite{chen2020asynchronous}, multiple node updates per iteration \cite{li2019convergence}, and stochasticity \cite{karimireddy2020scaffold}. However, most approaches require  convexity and smoothness \cite{li2019convergence,chen2020asynchronous}, making them unsuitable for MCP or SCAD penalized quantile regression, which is non-convex and non-smooth. Some methods, such as smooth federated learning techniques using continuous gradients \cite{davis2022proximal}, can address non-convex challenges but still require a smooth component, limiting their use in sparse penalized regression. Sub-gradient methods, designed for non-smooth objectives, often result in oscillatory and slow convergence  \cite{davis2019stochastic,mirzaeifard2023distributed}, reducing their effectiveness in identifying active coefficients essential for model sparsity and interpretability \cite{mirzaeifard2023smoothing}. Proximal gradient methods represent a significant advance, combining gradient steps with proximal updates and an increasing penalty parameter to better handle non-smooth objectives \cite{gao2024stochastic}. Despite their potential for managing non-convexity and distinguishing between zero and non-zero coefficients, these methods still lack reliable assurances of consistent reductions in each iteration, failing to guarantee continuous improvement in model accuracy. Given these shortcomings, there is a pressing need for a federated learning method tailored for non-convex environments. Such a method would effectively identify active and non-active coefficients and ensure substantial model accuracy improvements in each iteration, meeting the demands for high precision and robustness in decentralized IoT settings.

This paper introduces the \textbf{F}ederated \textbf{S}moothing \textbf{P}roximal \textbf{G}radient (FSPG) approach to address the challenges of federated quantile regression involving non-convex and non-smooth sparse penalties, such as MCP and SCAD, without requiring data sharing with a centralized node. The FSPG algorithm effectively manages non-convex optimization issues by employing smoothing techniques that do not require a smooth component or an implicit Lipschitz condition. Our approach transforms the non-smooth, non-convex problem into a sequence of smooth approximations, thus promoting efficient and robust convergence. The FSPG method utilizes an upper-bound smooth function to approximate the quantile regression function, combining the MM framework \cite{peng2015iterative} with smooth proximal gradient descent optimization \cite{davis2022proximal}. This strategy ensures a significant reduction in each update step, with the convergence of this smooth function closely approximating the quantile regression function. As the smoothness of the approximation diminishes with each iteration, we implement a time-varying increasing penalty parameter to maintain control over the process. Additionally, we incorporate a proximal update step at the centralized node, which is analogous in complexity to averaging and serves as a shared step for each node, thereby reducing complexity at the local level. In summary, the contributions of this paper are outlined as follows:

\begin{itemize}
\item We introduce FSPG, a federated smoothing gradient descent algorithm for sparse-penalized quantile regression with non-convex penalties, specifically targeting MCP and SCAD.
\item We address the challenges posed by the non-convex and non-smooth characteristics of the quantile regression function and penalties through an iterative smoothing process, enabling convergence under mild conditions.
\item We provide a comprehensive analysis of the FSPG, demonstrating a convergence rate of $o\left(k^{-\frac{1}{2}+\frac{d}{2}}\right)$ for the sub-gradient and $o\left(k^{-\frac{1}{2}-\frac{d}{2}}\right)$ for the norm of the coefficients' fluctuations over successive iterations, where $d \in (0,1)$ represents the power of the updating factor for the penalty parameter.

\item Through extensive simulations, we validate the enhanced performance of our proposed algorithm compared to existing methods in terms of accuracy and convergence.
\end{itemize}

The remainder of this paper is organized as follows: Section \ref{sec2} outlines the necessary preliminaries for our study. In Section \ref{sec3}, we present our proposed algorithms. The convergence proof is provided in Section \ref{sec4}. Section \ref{sec5} is dedicated to numerical simulations to validate the performance and effectiveness of our algorithm. Lastly, conclusions and discussions on future work are presented in Section \ref{sec6}.

\noindent\textit{\textbf{Mathematical Notations}}: Lowercase letters represent scalars, bold lowercase letters represent column vectors, and bold uppercase letters represent matrices. The transpose of a matrix is represented by $(\cdot)^\top$, and the $j$th column of a matrix $\mathbf{A}$ is denoted as $\mathbf{a}_{j}$. Additionally, the entry in the $i$th row and $j$th column of $\mathbf{A}$ is denoted as $a_{ij}$. Lastly, $\partial f(u)$ represents the sub-gradient of the function $f(\cdot)$ evaluated at $u$.

\section{Preliminaries}\label{sec2}
This section defines concepts and notations necessary for deriving the federated smoothing proximal gradient (FSPG) algorithm in Section III. In particular, we briefly review sparse quantile regression utilizing non-convex penalties for sparsity and outlier handling, and smoothing approximations aiding gradient-based optimization for non-smooth objectives.
\subsection{Sparse Quantile Regression Framework}
Consider a scalar variable $Y$ and a predictor vector $\mathbf{x}$ of dimension $P$. The conditional cumulative distribution function is defined as $F_{Y}(y|\mathbf{x}) = P(Y \leq y | \mathbf{x})$. For a given $\tau \in (0,1)$, the $\tau$th conditional quantile, $Q_Y(\tau|\mathbf{x})$, is given by $Q_Y(\tau|\mathbf{x}) = \inf\{y : F_Y(y|\mathbf{x}) \geq \tau\}$. The linear model for quantile regression relates $Q_Y(\tau|\mathbf{x})$ to $\mathbf{x} \in \mathbb{R}^P$ as follows \cite{koenker1982robust}:

\begin{equation}\label{eq1}
Q_Y(\tau|\mathbf{x}) = \mathbf{x}^\top \boldsymbol{\beta}_{\tau} + q_{\tau}^{\epsilon},
\end{equation}
where $\boldsymbol{\beta}{\tau} \in \mathbb{R}^P$ represents the coefficients of the regression model, and $q_{\tau}^{\epsilon} \in \mathbb{R}$ is the $\tau$th quantile of the error term, both being unknown and require estimation.

With a dataset consisting of $n$ pairs $\left\{\mathbf{x}_i, y_i\right\}_{i=1}^{n}$ and a chosen $\tau$, we can estimate the model parameters by solving the following optimization problem \cite{koenker1982robust}:

\begin{equation}\label{eq2}
\hat{\mathbf{w}} = \argmin_{\mathbf{w}} \frac{1}{n} \sum_{i=1}^{n} \rho_{\tau}\mathopen{}\left(y_i - \bar{\mathbf{x}}_i^\top \mathbf{w}\right)\mathclose{},
\end{equation}
where $\bar{\mathbf{x}}_j = \mathclose{}\left[\mathbf{x}_j^\top, 1\right]^\top\mathclose{} \in \mathbb{R}^{P+1}$  is the augmented input (feature) vector including a bias term, and $\mathbf{w} = \left[\boldsymbol{\beta}_{\tau}^\top, q_{\tau}^{\epsilon}\right]^\top \in \mathbb{R}^{P+1}$ encapsulates the regression parameters and the intercept, while function $\rho_{\tau} (u) = u\mathopen{}\left(\tau - I\mathopen{}\left(u < 0\right)
\mathclose{}\right)\mathclose{}$ is the check loss function with $I(\cdot)$ being the indicator function \cite{koenker1982robust}.

To enhance inference by incorporating model coefficient priors, we integrate a penalty function $P_{\lambda,\gamma} (\mathbf{w})$, transforming the optimization problem in \eqref{eq2} to the following form \cite{mirzaeifard2022admm}:

\begin{equation}\label{eq3}
\hat{\mathbf{w}} = \argmin_{\mathbf{w}} \frac{1}{n} \sum_{i=1}^{n} \rho_{\tau}\mathopen{}\left(y_i - \bar{\mathbf{x}}_i^\top \mathbf{w}\right)\mathclose{} + P_{\lambda,\gamma}(\mathbf{w}),
\end{equation}
which can be alternatively expressed as:
\begin{multline}\label{eq4}
\hat{\mathbf{w}} = \argmin_{\mathbf{w}}\frac{1}{2} \left\|\mathbf{y} - \bar{\mathbf{X}}^\top\mathbf{w}\right\|_1 \\+ \left(\tau - \frac{1}{2}\right)\mathbf{1}_n^\top\left(\mathbf{y} - \bar{\mathbf{X}}^\top\mathbf{w}\right) + n P_{\lambda,\gamma} (\mathbf{w}).
\end{multline}

In the federated setting, we consider a network consisting of $L$ edge devices (clients), each with its local dataset $\{\bar{\mathbf{X}}^{(l)}, \mathbf{y}^{(l)}\}$, where $\mathbf{y}^{(l)} \in \mathbb{R}^{M_l}$ is a column vector of responses and $\bar{\mathbf{X}}^{(l)} \in \mathbb{R}^{M_l \times (P+1)}$ is a matrix presenting the extended feature vectors, including a bias term. We assume $n=\sum_{l=1}^{L}M_l$, $\bar{\mathbf{X}}=\left[(\bar{\mathbf{X}}^{(1)})^{\top},\cdots,(\bar{\mathbf{X}}^{(L)})^{\top}\right]^{\top}$, and $\mathbf{y}=\left[(\mathbf{y}^{(1)})^{\top},\cdots,(\mathbf{y}^{(L)})^{\top}\right]^{\top}$. Thus,  the objective function \eqref{eq4} can thus be rewritten as:

\begin{equation}\label{eq5}
\min_{\mathbf{w}} \sum_{l=1}^{L} g_l (\mathbf{w}) + n P_{\lambda,\gamma} (\mathbf{w}),
\end{equation}
where $g_l (\mathbf{w}) = \frac{1}{2}\mathopen{}\left\|\mathbf{y}^{(l)} - (\bar{\mathbf{X}}^{(l)})^\top\mathbf{w}\right\|_1\mathclose{} + \mathopen{}\left(\tau - \frac{1}{2}\right)\mathclose{}\mathbf{1}_{M_l}^\top\mathopen{}\left(\mathbf{y}^{(l)} - \mathopen{}\left(\bar{\mathbf{X}}^{(l)}\right)^\top\mathbf{w}\mathclose{}\right)\mathclose{}$ represents the local loss function for client $l$.

While the $l_1$ norm penalty is widely adopted for inducing sparsity, it often leads to biased estimates. To address this issue, our approach utilizes the minimax concave penalty (MCP) \cite{zhang2010nearly} and the smoothly clipped absolute deviation (SCAD) \cite{fan2001variable}  as alternatives. These penalties are formulated as $P_{\lambda,\gamma}\mathopen{}\left(\mathbf{w}\right)\mathclose{}=\sum_{p=1}^P g_{\lambda,\gamma}\mathopen{}\left(w_p\right)\mathclose{}$, which help to reduce this bias and enhance model sparsity. These penalties are defined as follows, with the constraints that $\gamma\geq 1$ for MCP and $\gamma\geq 2$ for SCAD \cite{fan2001variable,zhang2010nearly}:

\begin{equation}\label{eq6}
 g_{\lambda,\gamma}^{\text{MCP}}\mathopen{}\left(w_p\right)\mathclose{}=
\begin{cases}
\lambda|w_p|-\frac{w_p^2}{2\gamma}, & |w_p| \leq \gamma \lambda 
\\
\frac{\gamma\lambda^2}{2}, & |w_p|> \gamma\lambda   
\end{cases}
\end{equation}
and 
\begin{equation}\label{eq7}
 g_{\lambda,\gamma}^{\text{SCAD}}\mathopen{}\left(w_p\right)\mathclose{}=
\begin{cases}
\lambda|w_p|, & |w_p| \leq \lambda 

\\
-\frac{|w_p|^2-2\gamma\lambda|w_p|+\lambda^2}{2\mathopen{}\left(\gamma-1\right)\mathclose{}}, & \lambda < |w_p| \leq \gamma\lambda 
\\
\frac{\mathopen{}\left(\gamma+1\right)\mathclose{}\lambda^2}{2}. &  |w_p| > \gamma\lambda 
\end{cases}
\end{equation}
These penalties effectively distinguish between active and inactive coefficients, which makes them highly suitable for sparse modeling. Notably, both MCP and SCAD exhibit weak convexity for specific $\rho$ values, ensuring that $g_{\lambda,\gamma}\mathopen{}\left(w\right)\mathclose{}+\frac{\rho}{2} w^2$ remains convex \cite{varma2019vector}. This aspect, detailed in prior work, highlights the nuanced efficiency of these functions in the context of sparse quantile regression \cite{mirzaeifard2023distributed,mirzaeifard2023smoothing}.

The proximal gradient method is a powerful optimization algorithm that enables optimization in a central machine without access to local data. However, conventional proximal gradient algorithms in the non-convex setting typically assume the presence of a smooth component \cite{davis2022proximal}. Modifications such as a time-varying penalty parameter have been proposed to extend proximal gradient methods to non-smooth functions \cite{gao2024stochastic}. However, these approaches do not guarantee that each iteration will monotonically improve the objective function value. This limitation underscores the need for a novel federated proximal gradient algorithm that is better suited to the non-smooth and federated setting of federated penalized quantile regression \cite{gao2024stochastic}. In this context, a smoothing approach can be practical to ensure improvement in each iteration.

\subsection{Smoothing Approximation}
Handling non-smooth functions in optimization presents significant challenges, primarily as we cannot exploit the beneficial properties of smooth functions, such as gradient-based convergence guarantees. A common strategy to overcome this limitation is to employ smoothing techniques, which involve replacing non-smooth functions with smooth approximations. These smoothed functions are easier to optimize, particularly within the proximal gradient framework we propose in this paper.

We begin by introducing a smoothing function, which approximates a non-smooth function $g$ with a smooth surrogate $\tilde{g}$, thereby facilitating a more streamlined optimization trajectory.
\begin{customdefinition}{1\cite{chen2012smoothing}}  Let $\tilde{g} : \Omega \subseteq \mathbb{R}^n \times \mathopen{}\left(0, +\infty\right)\mathclose{} \to \mathbb{R}$ serve as a smoothing approximation of $g$, with $g : \Omega \subseteq \mathbb{R}^n \to \mathbb{R}$ exhibits local Lipschitz continuity. The smoothing function $\tilde{g}$ satisfies the following properties:
\begin{enumerate}
\item Differentiability: $\tilde{g} (\cdot, \mu)$ is continuously differentiable over $\mathbb{R}^m$ for any $\mu > 0$. Additionally, for every $\mathbf{x} \in \Omega$, $\tilde{g} (x, \cdot)$ is differentiable over $(0, +\infty]$.
\item Convergence Criterion: For every $\mathbf{x} \in \Omega$,  $\lim_{\mu \to 0^+} \tilde{g} (\mathbf{x}, \mu) = g(x)$.
\item Bound on Gradient: There exists a positive constant $\kappa_{\tilde{g}}$ such that $|\nabla_\mu \tilde{g} (\mathbf{x}, \mu)| \leq \kappa_{\tilde{g}}$ for all $\mu \in (0, +\infty)$ and $\mathbf{x} \in \Omega$.
\item Gradient Convergence: The condition $\lim_{\mathbf{z} \to \mathbf{x}, \mu \to 0} \nabla_z \tilde{g} (\mathbf{z}, \mu) \subseteq \partial g(\mathbf{x})$ is satisfied.
Moreover, for each $\mathbf{x} \in \mathbb{R}^n$, the smoothing function $\tilde{g}$ upholds:
\item Universal Convergence: The relation $\lim_{\mathbf{z} \to \mathbf{x}, \mu \to 0} \tilde{g} (\mathbf{z}, \mu) = g(\mathbf{x})$ is maintained.
\item Lipschitz Continuity in $\mu$: A constant $L > 0$ exists such that $|\tilde{g} (\mathbf{x}, \mu_1) - \tilde{g} (\mathbf{x}, \mu_2)| \leq L|\mu_1 - \mu_2|$.
\item Gradient's Lipschitz Continuity: Given the convexity of $\tilde{g} (\mathbf{x}, \mu)$, a constant $l_\mu > 0$ ensures $\left\|\nabla\tilde{g} (\mathbf{x}, \mu) - \nabla \tilde{g} (\mathbf{y}, \mu) \right\|  \leq l_\mu \left\|\mathbf{x} - \mathbf{y}\right\|$ for any $\mathbf{x}, \mathbf{y} \in \Omega$. 

\end{enumerate}
\end{customdefinition}
Moreover, it has been shown in \cite{bian2013neural} that given smoothing approximations $\tilde{g}_1, \ldots, \tilde{g}_n$ for functions ${g}_1, \ldots, {g}_n$, where each $a_i \geq 0$ and every $g_i$ is regular, the linear combination $\sum_{i=1}^{n} a_i \tilde{g}i$ serves as a smoothing function for $\sum_{i=1}^{n} a_i {g}_i$.

The smoothing function $\tilde{g}$, as characterized above, introduces desirable features such as smoothness and comprehensive convergence. These attributes significantly reduce the complexity of the optimization problem, thereby enhancing our understanding of the convergence behavior. The smoothing gradient framework, based on this approach, was initially proposed in \cite{chen2012smoothing} for convex problems. Recently, it has been extended to matrix rank minimization applications, utilizing some inner loops but without a detailed analysis of the convergence rates \cite{yu2022smoothing}. In the subsequent section, we will explore how the concept of upper-bound smoothing approximations can be effectively applied within the single-loop proximal gradient approach to overcome non-smoothness challenges. 

\section{Federated Smoothing Proximal Gradient for Penalized Quantile Regression}\label{sec3}
In order to address the challenges arising from the lack of Lipschitz differentiability in our objective function, we here introduce the federated smoothing proximal gradient (FSPG) algorithm. This algorithm incorporates a dynamically increasing penalty parameter to enhance convergence properties. The core of FSPG involves approximating the non-smooth term $\left\|\mathbf{y}^{(l)} - (\bar{\mathbf{X}}^{(l)})^\top\mathbf{w}\right\|_1$ with a sum of smooth functions, following the smoothing approach described in \cite{chen2012smoothing}. Specifically, we employ a smoothing approximation function for each term $\left|{y}^{(l)}_{i} - (\bar{\mathbf{x}}^{(l)}_{i})^\top\mathbf{w}\right|$, defined as:
\begin{multline}\label{eq9}
f\mathopen{}\left({y}^{(l)}_{i} - \left(\bar{\mathbf{x}}^{(l)}_{i}\right)^\top\mathbf{w},\mu\right)\mathclose{}=\\
\begin{cases}
\left|{y}^{(l)}_{i} - (\bar{\mathbf{x}}^{(l)}_{i})^\top\mathbf{w}\right|, &  \mu \leq \left|{y}^{(l)}_{i} - \left(\bar{\mathbf{x}}^{(l)}_{i}\right)^\top\mathbf{w}\right|
\\
\frac{\left({y}^{(l)}_{i} - (\bar{\mathbf{x}}^{(l)}_{i})^\top\mathbf{w}\right)^2}{2\mu}+\frac{\mu}{2}. &  \left|{y}^{(l)}_{i} - \left(\bar{\mathbf{x}}^{(l)}_{i}\right)^\top\mathbf{w}\right| < \mu 
\end{cases}
\end{multline}
To aggregate these approximations, we introduce the function $h(\cdot)$, which sums up the smooth approximations for all elements in each local dataset. This results in the expression:
\begin{multline}\label{eq10}
h\mathopen{}\left(\mathbf{y}_l - (\bar{\mathbf{X}}^{(l)})^\top\mathbf{w},\mu\right)\mathclose{}=\frac{1}{2}\sum_{i=1}^{M_l} f\mathopen{}\left({y}^{(l)}_{i} - (\bar{\mathbf{x}}^{(l)}_{i})^\top\mathbf{w},\mu\right)\mathclose{} \\+ (\tau - \frac{1}{2})\mathbf{1}_{M_l}^\top(\mathbf{y}^{(l)} - (\bar{\mathbf{X}}^{(l)})^\top\mathbf{w}).
\end{multline}

Using the approximation provided, we reformulate the objective function as follows:
\begin{equation}\label{eq11}
\min_{\mathbf{w}}  \sum_{l=1}^{L} \tilde{g}_l (\mathbf{w},\mu) + n P_{\lambda,\gamma} (\mathbf{w}),
\end{equation}
where $\tilde{g}_l (\mathbf{w},\mu) =  h(\mathbf{y}^{(l)} - (\bar{\mathbf{X}}^{(l)})^\top\mathbf{w},\mu)$ represents the smoothed local loss function for client $l$. The approximation becomes increasingly precise as $\mu$ approaches zero, aligning the approximate objective function with the true objective. 

To iteratively refine the approximation, we update the parameters $\mu$ and $\sigma$ at each iteration using the following rules, where $c>0$, $\beta>0$, and $d \in (0,1)$:
 \begin{equation}\label{eq12}
     \sigma^{\mathopen{}\left(k+1\right)\mathclose{}}=c(k+1)^{d},
 \end{equation}
 and 
  \begin{equation}\label{eq13}
     \mu^{\mathopen{}\left(k+1\right)\mathclose{}}=\frac{\beta}{(k+1)^{d}}.
 \end{equation}
The above update rules balance smoothness and precision. As the algorithm advances, $\mu$ decreases, aligning the smoothed objective function more closely with the original non-smooth function, enhancing precision. Simultaneously, $\sigma$ increases, stabilizing iterative updates with a stronger penalty and improving convergence properties.

 Subsequently, each client $l$ computes its gradient with respect to $\mathbf{w}^{(k)}$ as follows:
 \begin{multline}\label{eq14}
     \nabla \tilde{g}_l\mathopen{}\left(\mathbf{w}^{(k)}\right)\mathclose{}= \\ \frac{1}{2}  \sum_{i=1}^{M_l} \nabla f\mathopen{}\left({y}^{(l)}_{i} - (\bar{\mathbf{x}}^{(l)}_{i})^\top\mathbf{w}^{(k)},\mu\right)\mathclose{} - (\tau-\frac{1}{2})\bar{\mathbf{X}}^{(l)}\mathbf{1}_{M_l},
 \end{multline}
 where 
 \begin{multline}\label{eq15}
\nabla f\mathopen{}\left({y}^{(l)}_{i} - \left(\bar{\mathbf{x}}^{(l)}_{i}\right)^\top\mathbf{w}^{(k)},\mu\right)\mathclose{}=-\bar{\mathbf{x}}^{(l)}_{i} \hspace{2mm}\times\\
\begin{cases}\operatorname{sgn}\mathopen{}\left({y}^{(l)}_{i} - \left(\bar{\mathbf{x}}^{(l)}_{i}\right)^\top\mathbf{w}^{(k)}\right)\mathclose{}, &  \mu \leq \left|{y}^{(l)}_{i} - (\bar{\mathbf{x}}^{(l)}_{i})^\top\mathbf{w}^{(k)}\right|
\\\mathopen{}\left(
\frac{{y}^{(l)}_{i} - \left(\bar{\mathbf{x}}^{(l)}_{i}\right)^\top\mathbf{w}}{\mu}\right)\mathclose{}. &  \left|{y}^{(l)}_{i} - (\bar{\mathbf{x}}^{(l)}_{i})^\top\mathbf{w}^{(k)}\right| < \mu 
\end{cases}
\end{multline}
 
 Following the local gradient updates, each client communicates this gradient to the central server. The server aggregates these gradients to update the global model parameter $\mathbf{w}$ using the following formula:
\begin{multline}\label{eq16}
   \mathbf{w}^{(k+1)} =\argmin_{\mathbf{w}} \sum_{l=1}^{L}  \left\langle \nabla (\tilde{g}_l\left(\mathbf{w}^{(k)}),\mu^{(k+1)}\right),\mathbf{w}-\mathbf{w}^{(k)}\right\rangle \\+\tilde{g}_l\mathopen{}\left(\mathbf{w}^{(k)}\right)\mathclose{}+\hspace{0.2mm}n \hspace{1mm} P_{\lambda,\gamma}\mathopen{}\left(\mathbf{w}\right)\mathclose{}+ \frac{\sigma^{(k+1)}}{2} \left\|\mathbf{w}-\mathbf{w}^{(k)}\right\|_2^2,
\end{multline}
where each component $w_p$ of $\mathbf{w}$ is individually updated via the proximal operator to maintain sparsity:
\begin{align}\label{eq17}\nonumber
w_p^{\mathopen{}\left(k+1\right)\mathclose{}}&= \argmin_{w_p}  n g_{\lambda,\gamma}\mathopen{}\left(w_p\right)\mathclose{}+\frac{\sigma_{\Psi}^{\mathopen{}\left(k+1\right)\mathclose{}}\mathopen{}}{2}\mathopen{}\left\|w_p-a_p  \right\|_2^2\mathclose{}, \\
&= \textbf{Prox}_{ g_{\lambda,\gamma}}\mathopen{}\left(a_p;\frac{ n}{\sigma_{\Psi}^{\mathopen{}\left(k+1\right)\mathclose{}}}\right)\mathclose{},
\end{align}
with $\mathbf{a}=\mathbf{w}^{(k)}-\frac{\sum_{l=1}^{L}\nabla \tilde{g}_i(\mathbf{w}^{(k)})}{\sigma^{(k+1)}}$. 

Finally, the server then disseminates $\mathbf{w}^{(k+1)}$ back to the clients for the subsequent iteration.

This iterative process is summarized in Algorithm \ref{alg:1}, which outlines the FSPG algorithm for federated sparse-penalized quantile regression. This approach effectively addresses the challenges associated with non-smooth objective functions, leveraging the benefits of smoothing approximations and proximal gradients within a federated learning paradigm.

\begin{algorithm}[t]
 \caption{Federated Smoothing Proximal Gradient for Penalized Quantile Regression (FSPG)}
 \label{alg:1}
\SetAlgoLined
Initialize $\mathbf{w}^{(0)}$, $c$ and $\beta$ such that $\beta c \geq \frac{\lambda_{\max}\left(\bar{\mathbf{X}}^\top\bar{\mathbf{X}}\right)}{2}$ and $d\in (0,1)$ and regularized parameters $\gamma$ and $\lambda$.\;
 \For{$k=0,\cdots$}{
  Update $\sigma^{(k+1)}$ by   \eqref{eq12}\;
  Update $\mu^{(k+1)}$ by   \eqref{eq13}\;
\For{$l=1,\cdots,L$}{
  Receive $\mathbf{w}^{(k)}$ from server\;
   Update $\nabla \tilde{g}_l\left(\mathbf{w}^{(k+1)},\mu^{(k+1)}\right)$ by \eqref{eq14} and \eqref{eq15}\;
   Send $\tilde{g}_l(\mathbf{w}^{(k+1)})$ to the server\;
 }
 Update $\mathbf{w}^{(k+1)}$ by   \eqref{eq16} and \eqref{eq17}\;
 }
\end{algorithm}
\section{Convergence Proof}\label{sec4}
Establishing the convergence of the proposed proximal gradient algorithm requires validating four essential conditions: boundedness, sufficient descent, subgradient bound, and global convergence, as highlighted in \cite[Theorem 2.9]{attouch2013convergence}. To this end, we start by demonstrating the boundedness of the augmented Lagrangian, formalized in Lemma \ref{lem2} below.

\begin{lemma}\label{lem2}
The function $\Phi_\sigma\left(\mathbf{w},\mathbf{w}',\mu\right) = \sum_{l=1}^{L} \tilde{g}_l\mathopen{}\left(\mathbf{w}',\mu\right)\mathclose{} + n P_{\lambda,\gamma} (\mathbf{w})+\sigma\left\|\mathbf{w}-\mathbf{w}'\right\|_2^2$ is lower bounded.
\end{lemma}
\begin{proof}
Consider the function $\Phi_{\sigma'}\mathopen{}\left(\mathbf{w},\mathbf{w}',\mu\right)\mathclose{}$ defined as above. We will examine each term to establish the lower bound:
\begin{enumerate}
    \item The term $\sum_{l=1}^{L} \tilde{g}_l\mathopen{}\left(\mathbf{w}',\mu\right)\mathclose{}$ represents the sum of smoothed local loss functions across $L$ clients. By construction, the smoothed approximation of the quantile loss function $\tilde{g}_l$ is designed to be always greater than or equal to the actual quantile loss function, which in turn is non-negative. Therefore, this term is non-negative for all $\mathbf{w}$ and $\mu > 0$.
    \item The regularization term $n P_{\lambda,\gamma} (\mathbf{w})$ involves penalties based on MCP and SCAD, both of which are non-negative for any $\mathbf{w}$. These penalties are designed to induce sparsity while ensuring the overall term remains non-negative.
    \item The quadratic term $\sigma\left\|\mathbf{w}-\mathbf{w}'\right\|_2^2$ is non-negative, as it is the squared Euclidean norm of the difference between $\mathbf{w}$ and $\mathbf{w}'$, scaled by a positive constant $\sigma$. 
\end{enumerate}

By combining these observations, we conclude that $\Phi\left(\mathbf{w},\mathbf{w}',\mu\right)$, being the sum of non-negative terms, is itself non-negative. Thus, $\Phi\left(\mathbf{w},\mathbf{w}',\mu\right) \geq 0$ for all $\mathbf{w}, \mathbf{w}' \in \mathbb{R}^n$ and $\mu > 0$, establishing that $\Phi\left(\mathbf{w},\mathbf{w}',\mu\right)$ is lower bounded.
\end{proof}
Having established the lower boundedness of $\Phi_\sigma$, we now focus on the smoothness properties of the function $\tilde{g}_l$, which are crucial for ensuring stability convergence, as detailed in the following lemma. 
\begin{lemma}\label{lem3}
Given the function $\tilde{g}_l(\mathbf{w},\mu) =  h\left(\mathbf{y}^{(l)} - \left(\bar{\mathbf{X}}^{(l)}\right)^\top\mathbf{w},\mu\right)$, with $h(\cdot,\mu)$ as defined in \eqref{eq10}, the smoothness constant $Li$ of $\tilde{g}_l\left(\mathbf{w},\mu\right)$ with respect to $\mathbf{w}$ can be upper bounded by $\frac{\lambda_{\max}\left(\left(\bar{\mathbf{X}}^{(l)}\right)^\top\bar{\mathbf{X}}^{(l)}\right)}{2\mu}$, where $\lambda_{\max} (\cdot)$ denotes the maximum eigenvalue of the matrix. Similarly, the smoothness of the sum $\sum_{l=1}^{L} \tilde{g}_l\left(\mathbf{w}',\mu\right)$ is upper bounded by $\frac{\lambda_{\max}\mathopen{}\left(\bar{\mathbf{X}}^\top\bar{\mathbf{X}}\right)\mathclose{}}{2\mu}$.
\end{lemma}
\begin{proof}
The gradient of $\tilde{g}_l\left(\mathbf{w},\mu\right)$ with respect to $\mathbf{w}$ in the quadratic region of $f(\cdot,\mu)$ is given by $\frac{{y}^{(l)}_{i} - \left(\bar{\mathbf{x}}^{(l)}_{i}\right)^\top\mathbf{w}}{\mu} \cdot (-\bar{\mathbf{x}}^{(l)}_{i})$. The change in this gradient is influenced by the spectral properties of $\bar{\mathbf{X}}^{(l)}$, specifically the maximum eigenvalue $\lambda_{\max}\left(\left(\bar{\mathbf{X}}^{(l)}\right)^\top\bar{\mathbf{X}}^{(l)}\right)$ as it can be determined from Hessian function, and inversely proportional to the smoothing parameter $\mu$. Therefore, considering the maximum rate of change in the gradient across all data points, the smoothness constant $Li$ of $\tilde{g}_l(\mathbf{w},\mu)$ can be upper bounded by $\frac{\lambda_{\max}\left(\left(\bar{\mathbf{X}}^{(l)}\right)^\top\bar{\mathbf{X}}^{(l)}\right)}{2\mu}$, reflecting the combined effect of the data matrix's spectral norm and the smoothing parameter on the function's smoothness. Similarly, the smoothness of $\sum_{l=1}^{L} \tilde{g}_l\left(\mathbf{w}',\mu\right)$ can be derived as $\frac{\lambda_{\max}\mathopen{}\left(\bar{\mathbf{X}}^\top\bar{\mathbf{X}}\right)\mathclose{}}{2\mu}$.
\end{proof}
With the smoothness of $\tilde{g}_l(\cdot,\mu)$ established, we now consider the behavior of the  $\Phi_\sigma$ over successive iterations. The following lemma, Lemma \ref{lem4} below, provides a crucial bound that highlights the progression towards convergence.
\begin{lemma}\label{lem4}
Given the function $\Phi_{\sigma}\mathopen{}\left(\mathbf{w},\mathbf{w}',\mu\right)\mathclose{} = \sum_{l=1}^{L} \tilde{g}_l\left(\mathbf{w}',\mu\right) + n P_{\lambda,\gamma} (\mathbf{w})+\sigma\left\|\mathbf{w}-\mathbf{w}'\right\|_2^2$, and given that  $\beta c \geq \frac{\lambda_{\max}\mathopen{}\left(\bar{\mathbf{X}}^\top\bar{\mathbf{X}}\right)\mathclose{}}{2}$, the difference in $\Phi_{\sigma}$ at successive iterations is bounded as follows:
\begin{multline}\label{eq18}\Phi_{\sigma^{(k+1)}}\mathopen{}\left(\mathbf{w}^{(k+1)},\mathbf{w}^{(k+1)},\mu^{(k+1)}\right)\mathclose{} \\ - \Phi_{\sigma^{(k)}}\mathopen{}\left(\mathbf{w}^{(k)},\mathbf{w}^{(k)},\mu^{(k)}\right)\mathclose{} \\ \leq -\left(\sigma^{(k+1)}-n\rho\right)\left\|\mathbf{w}^{(k+1)}-\mathbf{w}^{(k)}\right\|^2_2,
\end{multline}
where $\rho$ is the weak convexity parameter of $P_{\lambda,\gamma}(\cdot)$.
\end{lemma}
\begin{proof}
The proof systematically addresses updates to $\sigma$, $\mu$, and $\mathbf{w}$, as well as the subsequent update to $\mathbf{w}'$.\\

\textit{Updating $\sigma$ and $\mu$:} The first step involves updating $\sigma$ from $\sigma^{(k)}$ to $\sigma^{(k+1)}$ and $\mu$ from $\mu^{(k)}$ to $\mu^{(k+1)}$. These updates impact the function due to the smoothness concerning $\mu$ and the non-decreasing nature of $\sigma$. For $\mathbf{w}$ and $\mathbf{w}'$ fixed at $\mathbf{w}^{(k)}$, this yields:
\begin{equation}\label{eq19}
\Phi_{\sigma^{(k+1)}}\mathopen{}\left(\mathbf{w}^{(k)},\mathbf{w}^{(k)},\mu^{(k+1)}\right)\mathclose{} - \Phi_{\sigma^{(k)}}\mathopen{}\left(\mathbf{w}^{(k)},\mathbf{w}^{(k)},\mu^{(k)}\right)\mathclose{} \leq 0.
\end{equation}

\textit{Updating $\mathbf{w}$:} The second step, with $\sigma^{(k+1)}$ and $\mu^{(k+1)}$ set, involves updating $\mathbf{w}$ from $\mathbf{w}^{(k)}$ to $\mathbf{w}^{(k+1)}$. This step leverages the smoothness of $\tilde{g}_l(\mathbf{w},\mu)$ and the $(\sigma^{(k+1)}-n\rho)$-strong convexity from the quadratic penalty. This update ensures a descent, bounded by:
\begin{multline}\label{eq20}
\Phi_{\sigma^{(k+1)}}\mathopen{}\left(\mathbf{w}^{(k+1)},\mathbf{w}^{(k)},\mu^{(k+1)}\right)\mathclose{} \\ - \Phi_{\sigma^{(k+1)}}\mathopen{}\left(\mathbf{w}^{(k)},\mathbf{w}^{(k)},\mu^{(k+1)}\right)\mathclose{}\\ \leq -\left(\sigma^{(k+1)}-n\rho\right)\left\|\mathbf{w}^{(k+1)}-\mathbf{w}^{(k)}\right\|^2_2.
\end{multline}

\textit{Updating $\mathbf{w}'$:} For the final step, updating $\mathbf{w}'$ to $\mathbf{w}^{(k+1)}$, we leverage the smoothness of $\tilde{g}_l\left(\mathbf{w},\mu^{(k+1)}\right)$ and the condition $\beta c \geq \frac{\lambda_{\max}\mathopen{}\left(\bar{\mathbf{X}}^\top\bar{\mathbf{X}}\right)\mathclose{}}{2}$:
\begin{multline}\label{eq21}
\sum_{l=1}^{L} \tilde{g}_l\mathopen{}\left(\mathbf{w}^{(k+1)},\mu^{(k+1)}\right)\mathclose{} \\ \leq \sum_{l=1}^{L} \bigg[ \left\langle \nabla \tilde{g}_l\mathopen{}\left(\mathbf{w}^{(k)},\mu^{(k+1)}\right)\mathclose{}, \mathbf{w}^{(k+1)}-\mathbf{w}^{(k)} \right\rangle \\
 + \tilde{g}_l\mathopen{}\left(\mathbf{w}^{(k)},\mu^{(k+1)}\right)\mathclose{} + \frac{\sigma^{(k+1)}}{2} \left\|\mathbf{w}^{(k+1)}-\mathbf{w}^{(k)}\right\|_2^2 \bigg].
\end{multline}

By combining these updates, we observe an overall non-positive change in $\Phi_{\sigma}$, confirming the descent property across iterations as \eqref{eq18}.
\end{proof}
Building upon the bounded difference in $\Phi_\sigma$, we now explore the sufficient descent property, which demonstrates the capability of the algorithm to iteratively reduce the objective function value. This property is formalized in Theorem \ref{theo1} below, providing insights into the convergence dynamics of the algorithm.
  \begin{theorem}[Sufficient Descent Property]\label{theo1}
Assume there exists a $K$ such that for all $k \geq K$, the condition  $\sigma_{\Psi}^{(k+1)} > { n \rho}$ holds, in which $\rho$ is the weak convexity parameter of $P_{\lambda,\gamma} (\cdot)$. Suppose we update $\sigma^{\mathopen{}\left(k+1\right)\mathclose{}}=c(k+1)^{d}$ with $c > 0$, and $d\in (0,1)$ and $\mu^{\mathopen{}\left(k+1\right)\mathclose{}} = \frac{\beta}{(k+1)^{d}}$, where $\beta c\geq {\lambda_{\max}\left(\bar{\mathbf{X}}^\top\bar{\mathbf{X}}\right)}$. Then, under the iterations of Algorithm \ref{alg:1}, the following relation holds:
\begin{align}\label{eq23}
\mathopen{}\left\|\mathbf{w}^{\mathopen{}\left(k+1\right)\mathclose{}} - \mathbf{w}^{\mathopen{}\left(k\right)\mathclose{}}\right\|_2^2\mathclose{} \in o\mathopen{}\left({k^{-1-d}}\right)\mathclose{}. 
\end{align}
 \end{theorem}
 \begin{proof}
     Consider $K'$ greater than $K$. The overall change in the augmented Lagrangian from iteration $K$ to iteration $K'$, using Lemma \ref{lem4} for single iteration, is given as:
\begin{multline}\label{eq24}
    \Phi_{\sigma^{(K')}}\mathopen{}\left(\mathbf{w}^{(K')},\mathbf{w}^{(K')},\mu^{(K')}\right)\mathclose{} - \Phi_{\sigma^{(K)}}\mathopen{}\left(\mathbf{w}^{(K)},\mathbf{w}^{(K)},\mu^{(K)}\right)\mathclose{}\\ \leq - S_{K'}.
\end{multline}
where 
\begin{equation}\label{eq25}
   S_{K'}=\sum_{k=K}^{K'-1} \eta^{(k+1)}\left\|\mathbf{w}^{(k+1)}-\mathbf{w}^{(k)}\right\|^2_2, 
\end{equation}
in which $\eta^{(k+1)}=\sigma^{(k+1)}-n\rho$.
Given that $\Phi_{\sigma}\mathopen{}\left(\mathbf{w},\mathbf{w}',\mu\right)\mathclose{}$ is lower-bounded by Lemma \ref{lem2} and since $\sigma^{(K)}\geq n\rho$, teach element of $S_{K'}$ is non-negative. Thus, we can conclude that $0 \leq \lim_{K'\rightarrow \infty} S_{K'}<\infty$. This implies $\lim \eta^{(k+1)}\left\|\mathbf{w}^{(k+1)}-\mathbf{w}^{(k)}\right\|^2_2=0$. As each
 $\eta^{\mathopen{}\left(k+1\right)\mathclose{}}$ is positive and increases at rate of $\Omega\mathopen{}\left(k^{d}\right)\mathclose{}$, it follows that $\mathopen{}\left\|\mathbf{w}^{\mathopen{}\left(k+1\right)\mathclose{}}-\mathbf{w}^{\mathopen{}\left(k\right)\mathclose{}}\right\|_2^2\mathclose{}$ decrease at a rate of $o\mathopen{}\left(k^{-1-d}\right)\mathclose{}$.
 \end{proof}
 \begin{remark}
Theorem \ref{theo1} provides the flexibility to update $\sigma_\Psi^{\mathopen{}\left(k+1\right)\mathclose{}}$ and $\mu^{\mathopen{}\left(k+1\right)\mathclose{}}$ differently for iterations where $k<K$.This flexibility can be advantageous in practice, allowing for tailored updates that may enhance algorithm performance.
\end{remark}
 After establishing sufficient descent property in Theorem \ref{theo1}, which assures that the algorithm makes significant progress towards the solution in each iteration, we further refine our understanding of the behavior of FSPG through Lemma \ref{lem5} below. In particular, Lemma \ref{lem5} provides insights on the iteration-to-iteration dynamics, specifically focusing on the changes in the gradient of the sum of functions $\tilde{g}_l (\cdot,\mu)$. This lemma is crucial because it bounds these gradient changes, which are indicative of stability and smoothness in navigating the solution space. 
 \begin{lemma}\label{lem5} When following Algorithm \ref{alg:1}, the change in the derivative of  $ \sum_{l=1}^{L} \tilde{g}_l (\mathbf{w},\mu)$  between successive iterations is bounded by:
\begin{multline}\label{eq26}
 \mathopen{}\left\|\nabla \sum_{l=1}^{L} \tilde{g}_l\left(\mathbf{w}^{(k+1)},\mu^{(k+2)}\right)- \nabla \sum_{l=1}^{L} \tilde{g}_l\left(\mathbf{w}^{(k)},\mu^{(k+1)}\right) \right\|_2\mathclose{} \\ \leq \frac{\lambda_{\max}\left(\bar{\mathbf{X}}^\top\bar{\mathbf{X}}\right)}{2\mu^{(k+1)}} \left\|\mathbf{w}^{\mathopen{}\left(k+1\right)\mathclose{}}-\mathbf{w}^{\mathopen{}\left(k\right)\mathclose{}}\right\|_2\mathclose{}\\+ \frac{n\sqrt{\lambda_{\max}\left(\bar{\mathbf{X}}^\top\bar{\mathbf{X}}\right)}}{2} \mathopen{}\left(\frac{\mu^{\mathopen{}\left(k+1\right)\mathclose{}}-\mu^{\mathopen{}\left(k+2\right)\mathclose{}}}{\mu^{\mathopen{}\left(k+2\right)\mathclose{}}}\right)\mathclose{}.
\end{multline}
\end{lemma}
\begin{proof}
To establish the bound on the change in the derivative of the function $\sum_{l=1}^{L} \tilde{g}_l(\mathbf{w},\mu)$ between successive iterations, we begin by examining the difference in gradients at iterations $k$ and $k+1$. 

First, consider the expression for the change in the gradient:
\begin{multline}\label{eq27}
 \mathopen{}\left\|\nabla \sum_{l=1}^{L} \tilde{g}_l\mathopen{}\left(\mathbf{w}^{(k+1)},\mu^{(k+2)}\right)\mathclose{}- \nabla \sum_{l=1}^{L} \tilde{g}_l\mathopen{}\left(\mathbf{w}^{(k)},\mu^{(k+1)}\right)\mathclose{} \right\|_2\mathclose{}  = \\ 
\bigg\| \nabla \sum_{l=1}^{L} \tilde{g}_l\mathopen{}\left(\mathbf{w}^{(k+1)},\mu^{(k+2)}\right)\mathclose{}- \nabla \sum_{l=1}^{L} \tilde{g}_l\mathopen{}\left(\mathbf{w}^{(k+1)},\mu^{(k+1)}\right)\mathclose{} \\+\nabla \sum_{l=1}^{L} \tilde{g}_l\mathopen{}\left(\mathbf{w}^{(k+1)},\mu^{(k+1)}\right)\mathclose{}- \nabla \sum_{l=1}^{L} \tilde{g}_l\mathopen{}\left(\mathbf{w}^{(k)},\mu^{(k+1)}\right)\mathclose{}  \bigg\|_2
\end{multline}
Next, we apply the triangle inequality to separate the terms:
\begin{multline}\label{eq28}
 \mathopen{}\left\|\nabla \sum_{l=1}^{L} \tilde{g}_l\mathopen{}\left(\mathbf{w}^{(k+1)},\mu^{(k+2)}\right)\mathclose{}- \nabla \sum_{l=1}^{L} \tilde{g}_l\mathopen{}\left(\mathbf{w}^{(k)},\mu^{(k+1)}\right)\mathclose{} \right\|_2\mathclose{}  \leq \\ 
\left\| \nabla \sum_{l=1}^{L} \tilde{g}_l\mathopen{}\left(\mathbf{w}^{(k+1)},\mu^{(k+2)}\right)\mathclose{}- \nabla \sum_{l=1}^{L}   \tilde{g}_l\left(\mathbf{w}^{(k+1)},\mu^{(k+1)}\right)\mathclose{}  \right\|_2\\+\left\|\nabla \sum_{l=1}^{L} \tilde{g}_l\mathopen{}\left(\mathbf{w}^{(k+1)},\mu^{(k+1)}\right)\mathclose{}- \nabla \sum_{l=1}^{L} \tilde{g}_l\left(\mathbf{w}^{(k)},\mu^{(k+1)}\right)\mathclose{} \right\|_2
\end{multline}
We now need to bound these two terms separately. 

For the first term, we leverage the smoothness property of the function $\tilde{g}_l(\mathbf{w},\mu)$, specifically the $\frac{\lambda_{\max}\left(\bar{\mathbf{X}}^\top\bar{\mathbf{X}}\right)}{2\mu^{(k+1)}}$-smoothness, which gives us:
\begin{multline}\label{eq29}
    \left\| \nabla \sum_{l=1}^{L} \tilde{g}_l\left(\mathbf{w}^{(k+1)},\mu^{(k+1)}\right)- \nabla \sum_{l=1}^{L}   \tilde{g}_l\left(\mathbf{w}^{(k)},\mu^{(k+1)}\right) \right\|_2 \\
    \leq \frac{\lambda_{\max}\left(\bar{\mathbf{X}}^\top\bar{\mathbf{X}}\right)}{2\mu^{(k+1)}} \left\|\mathbf{w}^{\mathopen{}\left(k+1\right)\mathclose{}}-\mathbf{w}^{\mathopen{}\left(k\right)\mathclose{}}\right\|_2\mathclose{}.
\end{multline}
For the second term, we utilize the result from \cite{mirzaeifard2023smoothing} (Lemma 6), which provides a bound on the difference in gradients with respect to the smoothing parameter $\mu$:
\begin{multline}\label{eq30}
  \mathopen{}\left\|\nabla  h\mathopen{}\left(\mathbf{z},{\mu^{\mathopen{}\left(k+1\right)\mathclose{}}}\right)\mathclose{} -\nabla  h\mathopen{}\left(\mathbf{z},{\mu^{\mathopen{}\left(k+2\right)\mathclose{}}}\right)\mathclose{}\right\|\mathclose{} \leq \\ \frac{n}{2} \mathopen{}\left(\frac{\mu^{\mathopen{}\left(k+1\right)\mathclose{}}-\mu^{\mathopen{}\left(k+2\right)\mathclose{}}}{\mu^{\mathopen{}\left(k+1\right)\mathclose{}}}\right)\mathclose{}, 
  \end{multline}
where $n$ is the number of samples.   Considering that $\nabla \sum_{l=1}^{L} \tilde{g}_l\left(\mathbf{w}^{(k+1)},\mu^{(k+1)}\right) = -\bar{\mathbf{X}} \nabla h(\mathbf{z},\mu^{(k+1)})$, where $\mathbf{z} = \mathbf{y} - \bar{\mathbf{X}}^{\top}\mathbf{w}^{(k+1)}$, and applying the Cauchy-Schwartz inequality, we obtain:
  \begin{multline}\label{eq31}
      \left\|\nabla \sum_{l=1}^{L} \tilde{g}_l\left(\mathbf{w}^{(k+1)},\mu^{(k+2)}\right)- \nabla \sum_{l=1}^{L} \tilde{g}_l\left(\mathbf{w}^{(k+1)},\mu^{(k+1)}\right)  \right\|_2 \\ 
      \leq   \frac{n \left\|\bar{\mathbf{X}}\right\|}{2} \mathopen{}\left(\frac{\mu^{\mathopen{}\left(k\right)\mathclose{}}-\mu^{\mathopen{}\left(k+1\right)\mathclose{}}}{\mu^{\mathopen{}\left(k\right)\mathclose{}}}\right)\mathclose{},   
  \end{multline}
  where $\|\bar{\mathbf{X}}\|=\sqrt{\lambda_{\max}\left(\bar{\mathbf{X}}^\top\bar{\mathbf{X}}\right)}$. Substituting \eqref{eq29} and \eqref{eq31} in \eqref{eq28} yields \eqref{eq26}.
\end{proof}
The bounded difference in successive gradients, as established by Lemma \ref{lem5}, ensures that the algorithm does not exhibit erratic behavior, which is essential for convergence to a local optimum.

Building on this, Lemma \ref{lem6} below provides further refinement by examining the behavior of the smoothness parameter $\mu$ across iterations. By establishing that the relative changes in $\mu$ decrease in a controlled manner (specifically, at the order of 
$O\left(k^{-1}\right)$, Lemma \ref{lem6} guarantees that the algorithm adjustments become increasingly subtle as it progresses. 
\begin{lemma}\label{lem6}
Given the update rule for  $\mu^{(k)}$ defined as $\mu^{(k)} = \frac{\beta}{k^d}$, where $\beta > 0$ and $d \in (0,1)$ is a constant, the relative difference between successive $\mu$ values, $\left\|\frac{\mu^{(k)} - \mu^{(k+1)}}{\mu^{(k)}}\right\|$, is of the order $O\left(k^{-1}\right)$.
\end{lemma}
\begin{proof}
Consider the continuous analog of $\mu^{(k)}$, defined as $\mu(k) = \frac{\beta}{k^d}$. The derivative of $\mu(k)$ with respect to $k$ is given by:
\[
\mu'(k) = -\frac{\beta d}{k^{d+1}}.
\]

The relative difference between $\mu^{(k)}$ and $\mu^{(k+1)}$, normalized by $\mu^{(k)}$, can be approximated by the derivative $\mu'(k)$ scaled by the change in $k$, which is 1 (from $k$ to $k+1$):
\[
\left\|\frac{\mu^{(k)} - \mu^{(k+1)}}{\mu^{(k)}}\right\| \approx \left| \frac{\mu'(k)}{\mu(k)} \right| = \left| -\frac{\beta d}{k^{d+1}} \cdot \frac{k^d}{\beta} \right| = \frac{d}{k}.
\]

Since $d \in (0, 1)$, this expression is well-defined for all $k > 0$. Therefore, the relative difference $\left\|\frac{\mu^{(k)} - \mu^{(k+1)}}{\mu^{(k)}}\right\|$ is of the order $O\left(k^{-1}\right)$.
\end{proof}
Lemma \ref{lem6} is pivotal in understanding the adaptability and fine-tuning of FSPG, ensuring that the updates facilitate convergence.

With the groundwork laid by Lemmas \ref{lem5} and \ref{lem6}, which ensure that both the gradient changes and the updates to the smoothness parameter $\mu$ are well-behaved, the following theorem, Theorem \ref{theo2}, introduces the concept of a sub-gradient bound. This theorem is pivotal because it transitions from the stability and smoothness of the FSPG (ensured by the lemmas) to establishing its effectiveness in approaching a solution.

\begin{theorem}[Sub-gradient Bound Property]\label{theo2}
Consider Algorithm \ref{alg:1} applied to the problem \eqref{eq5}. Suppose the condition $\beta c  \geq \frac{\lambda_{\max}\left(\bar{\mathbf{X}}^\top\bar{\mathbf{X}}\right)}{2}$ holds, and there exists a sufficiently large K such that $\sigma^{(k+1)} \geq n \rho$ for all $k \geq K$, where $\rho$ is the weak convexity parameter of $P_{\lambda,\gamma} (\cdot)$. Assuming the step-size $\sigma$ updates according to $\sigma^{(k+1)} = c (k+1)^{d}$, and the smoothness parameter $\mu$ updates via $\mu^{(k+1)} = \frac{\beta}{(k+1)^{d}}$, there exists a sequence $\boldsymbol{\kappa}^{(k+1)} \in \nabla \sum_{l=1}^{L} \tilde{g}_l\left(\mathbf{w}^{(k+1)},\mu^{(k+2)}\right) + n \partial P_{\lambda,\gamma}\left(\mathbf{w}^{(k+1)}\right)$ such that
\begin{equation}\label{eq32}
\left\|\boldsymbol{\kappa}^{(k+1)}\right\| \in o\left(k^{-\frac{1}{2}+\frac{d}{2}}\right).
\end{equation}
\end{theorem}

\begin{proof}
Given that $\beta c \geq \frac{\lambda_{\max}\left(\bar{\mathbf{X}}^\top\bar{\mathbf{X}}\right)}{2}$ and $\sigma^{(k+1)} \geq n \rho$ for $k \geq K$, the optimality conditions from the update step in \eqref{eq16} becomes:
\begin{multline}\label{eq33}
     \mathbf{0} \in  \sum_{l=1}^{L} \nabla \tilde{g}_l\left(\mathbf{w}^{(k)},\mu^{(k+1)}\right)+\sigma^{(k+1)}\left(\mathbf{w}^{(k+1)}-\mathbf{w}^{(k)}\right) \\ +n \partial P_{\lambda,\gamma}\left(\mathbf{w}^{(k+1)}\right).
\end{multline}
Rearranging, we obtain:
\begin{multline}\label{eq34}
    -\sum_{l=1}^{L} \nabla \tilde{g}_l\left(\mathbf{w}^{(k)},\mu^{(k+1)}\right)+\sigma^{(k+1)}\left(\mathbf{w}^{(k)}-\mathbf{w}^{(k+1)}\right) \\ \in n \partial P_{\lambda,\gamma}\left(\mathbf{w}^{(k+1)}\right).
\end{multline}
Adding $\sum_{l=1}^{L} \nabla \tilde{g}_l\left(\mathbf{w}^{(k+1)},\mu^{(k+2)}\right)$ to both sides of \eqref{eq34}, we get:
\begin{multline}\label{eq35}
     \boldsymbol{\kappa}^{(k+1)}  \in \sum_{l=1}^{L} \nabla \tilde{g}_l\left(\mathbf{w}^{(k+1)},\mu^{(k+2)}\right) + n \partial P_{\lambda,\gamma}\left(\mathbf{w}^{(k+1)}\right).
\end{multline}
where $\boldsymbol{\kappa}^{(k+1)} = \sum_{l=1}^{L} \nabla \tilde{g}_l\left(\mathbf{w}^{(k+1)},\mu^{(k+2)}\right) - \sum_{l=1}^{L} \nabla \tilde{g}_l\left(\mathbf{w}^{(k)},\mu^{(k+1)}\right) + \sigma^{(k+1)}\left(\mathbf{w}^{(k)}-\mathbf{w}^{(k+1)}\right)$. Applying the triangle inequality gives us:
\begin{multline}\label{eq36}
   \left\|\boldsymbol{\kappa}^{(k+1)} \right\|_2 \leq   \sigma^{(k+1)}\left\|\mathbf{w}^{(k+1)}-\mathbf{w}^{(k)}
    \right\|_2+ \\ \left\|\sum_{l=1}^{L} \nabla \tilde{g}_l\left(\mathbf{w}^{(k+1)},\mu^{(k+2)}\right) - \sum_{l=1}^{L} \nabla \tilde{g}_l\left(\mathbf{w}^{(k)},\mu^{(k+1)}\right)\right\|_2.
\end{multline}
Using the results from Lemma \ref{lem6} in \eqref{eq36}, we have:
\begin{multline}\label{eq37}
   \left\|\boldsymbol{\kappa}^{(k+1)} \right\|_2 \leq \\  \left(\sigma^{(k+1)}+\frac{\lambda_{\max}\left(\bar{\mathbf{X}}^\top\bar{\mathbf{X}}\right)}{2\mu^{(k+1)}}\right)\left\|\mathbf{w}^{(k+1)}-\mathbf{w}^{(k)}
    \right\|_2\\
     \\+ \frac{n\sqrt{\lambda_{\max}\left(\bar{\mathbf{X}}^\top\bar{\mathbf{X}}\right)}}{2} \mathopen{}\left(\frac{\mu^{\mathopen{}\left(k+1\right)\mathclose{}}-\mu^{\mathopen{}\left(k+2\right)\mathclose{}}}{\mu^{\mathopen{}\left(k+2\right)\mathclose{}}}\right)\mathclose{}.
\end{multline}

Considering the update rule for $\mu$ and $\sigma$, we observe that $\left(\sigma^{(k+1)}+\frac{\lambda_{\max}\left(\bar{\mathbf{X}}^\top\bar{\mathbf{X}}\right)}{2\mu^{(k+1)}}\right)$ increases at a rate of $\Omega(k^{d})$. According to Theorem \ref{theo1}, the term $\left\|\mathbf{w}^{(k+1)}-\mathbf{w}^{(k)}\right\|_2$ decreases at a rate of $o\left(k^{-\frac{1}{2}-\frac{d}{2}}\right)$. Thus, the product $\left(\sigma^{(k+1)}+\frac{\lambda_{\max}\left(\bar{\mathbf{X}}^\top\bar{\mathbf{X}}\right)}{2\mu^{(k+1)}}\right)\left\|\mathbf{w}^{(k+1)}-\mathbf{w}^{(k)}\right\|_2$ diminishes at a rate of $o\left(k^{-\frac{1}{2}+\frac{d}{2}}\right)$. Furthermore, from Lemma \ref{lem6}, we understand that the term $\frac{n\sqrt{\lambda_{\max}\left(\bar{\mathbf{X}}^\top\bar{\mathbf{X}}\right)}}{2} \left(\frac{\mu^{(k+1)}-\mu^{(k+2)}}{\mu^{(k+2)}}\right)$ decreases at a rate of $O\left(k^{-1}\right)$. Therefore, we can conclude that the norm $\left\|\boldsymbol{\kappa}^{(k+1)}\right\|_2$ decreases at a rate of $o\left(k^{-\frac{1}{2}+\frac{d}{2}}\right)$.
\end{proof}
\begin{remark}
The FSPG algorithm significantly improves the convergence rate compared to the FPG algorithm. Specifically, while the derivative of the Moreau envelope in sub-gradient algorithms and FPG indicates coefficient fluctuations and can be bounded by \(O\left(k^{-\frac{1}{4}}\right)\) \cite{davis2019stochastic}, the FSPG achieves a faster fluctuation rate bounded by \(o\left(k^{-\frac{3}{4}}\right)\). Moreover, the convergence rate of the sub-gradient of the augmented Lagrangian function in FSPG offers valuable insights into the optimality of the solution.
\end{remark}

\begin{remark}
There is a notable trade-off between the convergence rates measured by \(\left\|\mathbf{w}^{(k+1)} - \mathbf{w}^{(k)}\right\|^2\) and \(\left\|\boldsymbol{\kappa}^{(k+1)}\right\|\). Specifically, as the parameter \(d\) increases, we observe faster convergence in terms of \(\left\|\mathbf{w}^{(k+1)} - \mathbf{w}^{(k)}\right\|^2\) but a slower convergence for \(\left\|\boldsymbol{\kappa}^{(k+1)}\right\|\). 
\end{remark}

By showing that there exists a sequence within the sub-differential of the objective function that diminishes over time, Theorem \ref{theo2} essentially claims that not only does the algorithm progress smoothly, but it also does so in the right direction towards a point where the gradient can vanish, which is characteristic of optimality in non-convex optimization problems.

With a comprehensive understanding of the descent properties, gradient behaviors, and sub-gradient bounds, we are now in a position to synthesize these insights into a global perspective on the convergence of the FSPG.

\begin{figure*}[ht]
     \centering
     \begin{subfigure}[b]{0.245\textwidth}
         \centering
        \adjustbox{valign=t}{\includegraphics[width=49mm, height=37.5mm]{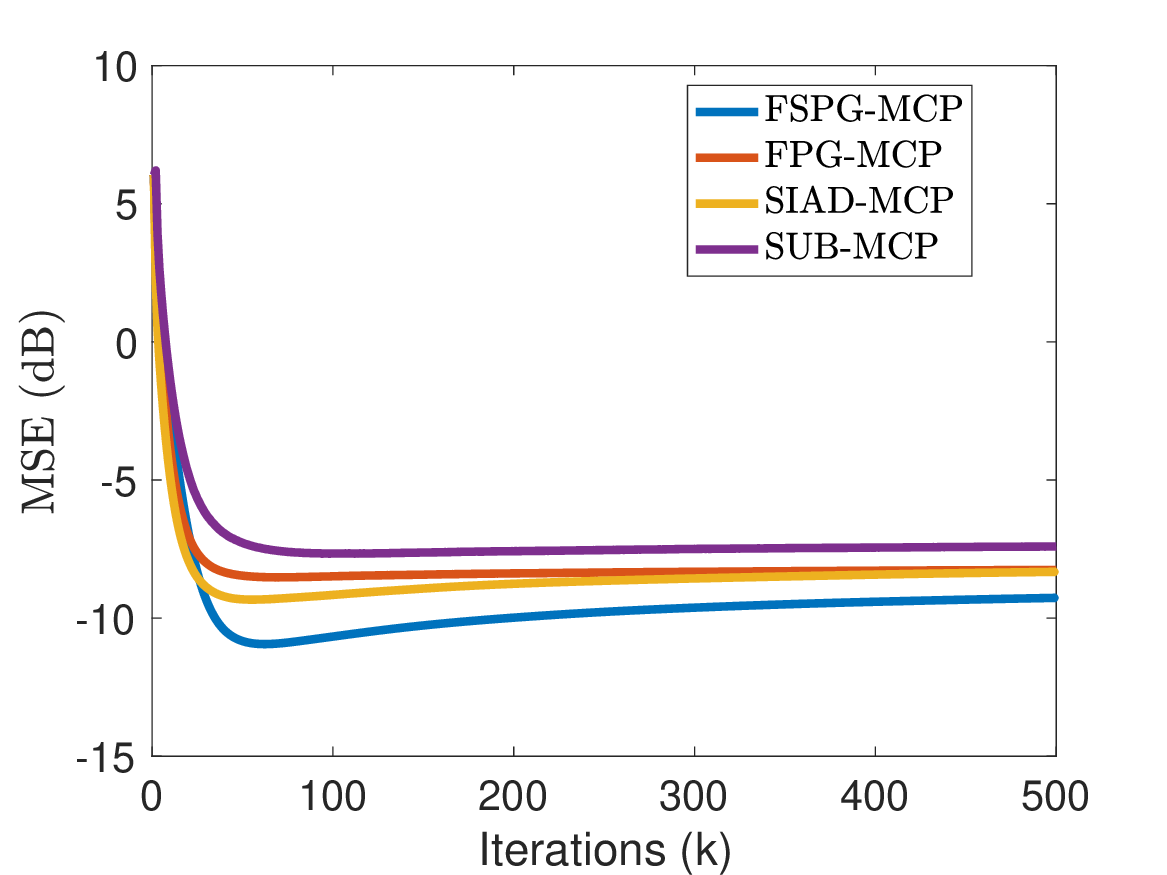}  \vspace{-1mm} }
      \caption{MCP ($\tau=0.55$)} 
     \end{subfigure}
     \hfill
     \begin{subfigure}[b]{0.245\textwidth}
         \centering  \adjustbox{valign=t}{\includegraphics[width=49mm, height=37.5mm]{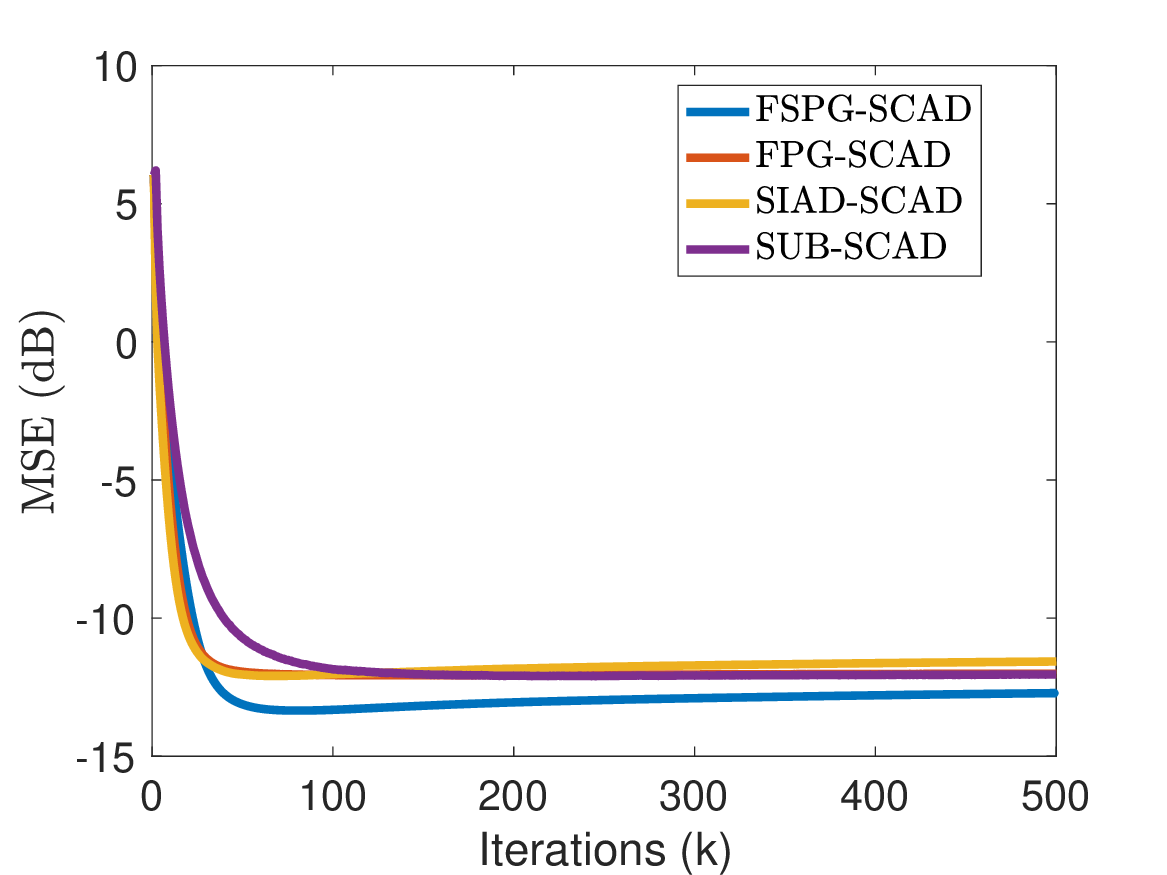}}
        \caption{SCAD ($\tau=0.55$)} 
     \end{subfigure}
    \centering
     \begin{subfigure}[b]{0.245\textwidth}
         \centering
        \adjustbox{valign=t}{\includegraphics[width=49mm, height=37.5mm]{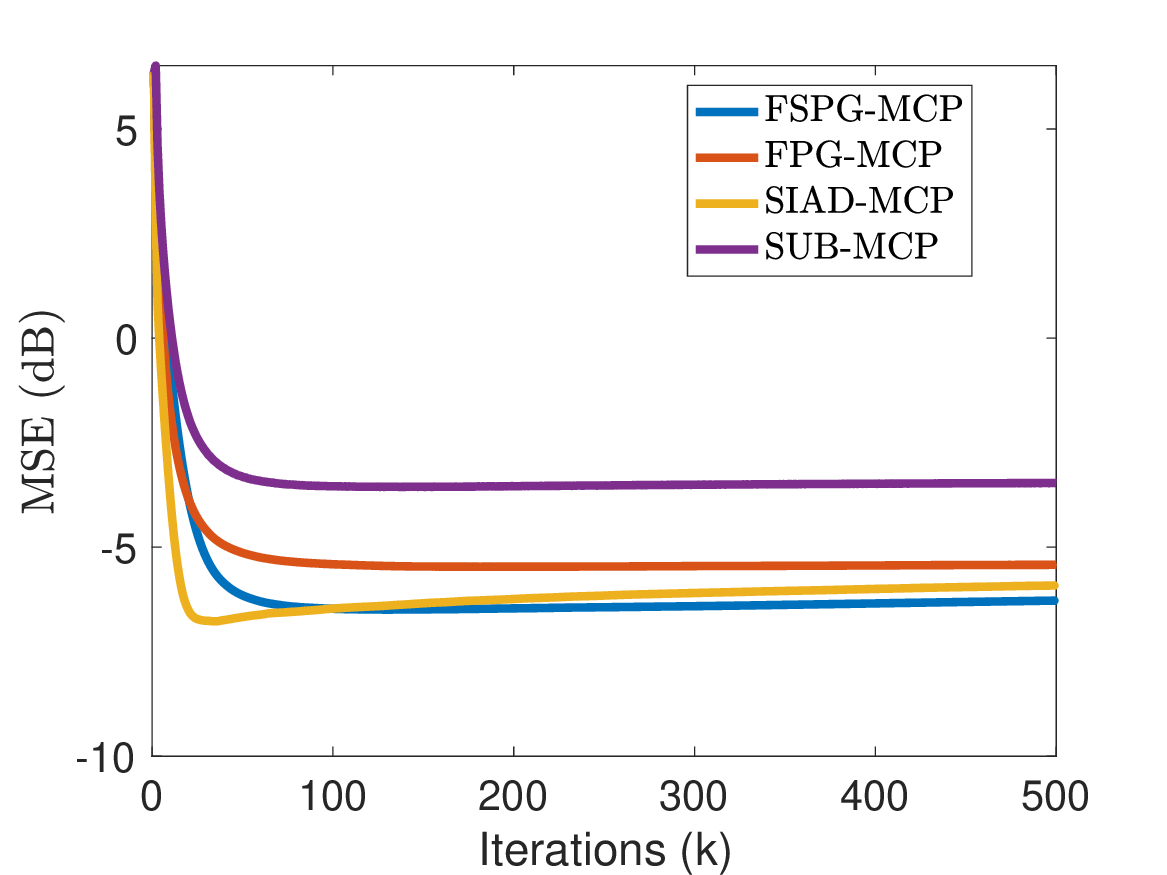}  \vspace{-1mm} }
      \caption{MCP ($\tau=0.7$)}   
     \end{subfigure}
     \hfill
     \begin{subfigure}[b]{0.245\textwidth}
         \centering  \adjustbox{valign=t}{\includegraphics[width=49mm, height=37.5mm]{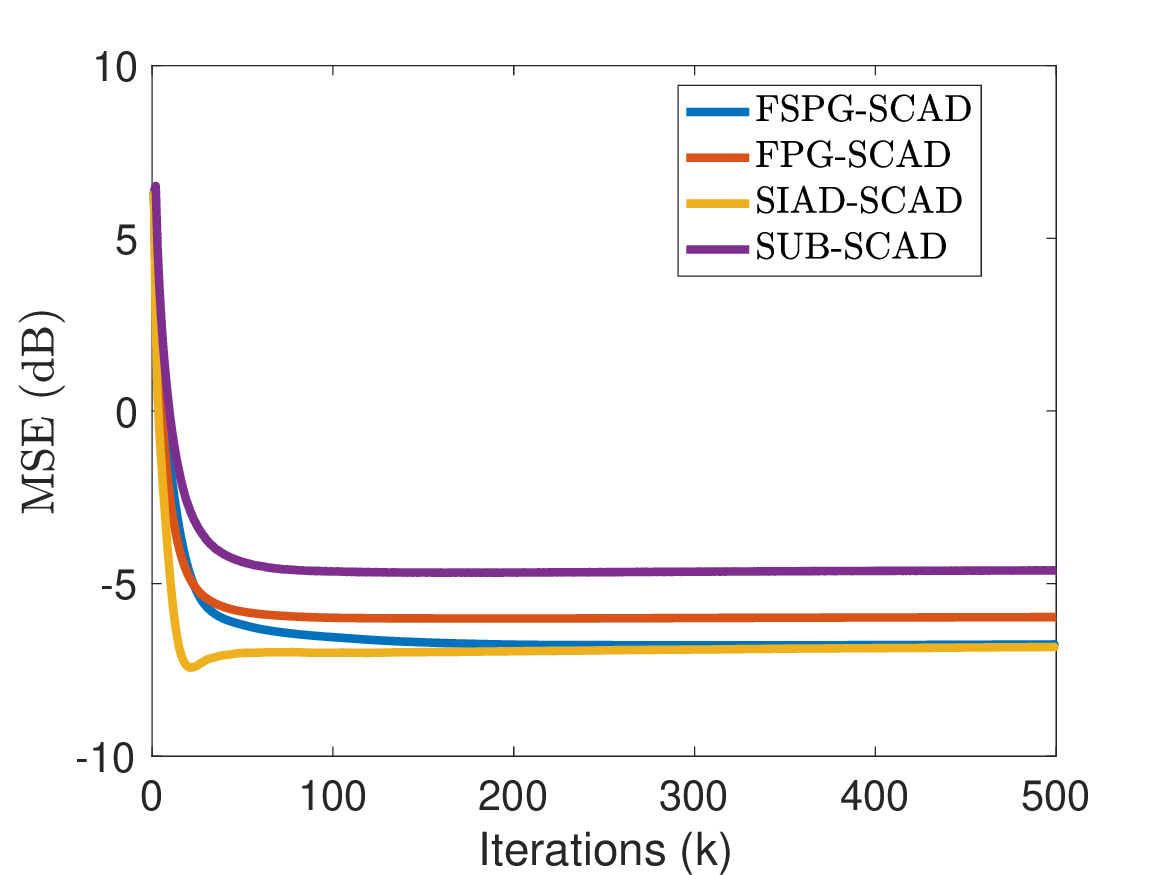}}
        \caption{SCAD ($\tau=0.7$)} 
     \end{subfigure}
     \centering
   \caption{MSE versus iterations.}
   \label{fig1}
\end{figure*}

\begin{figure*}[ht]
     \centering
       \begin{subfigure}[b]{0.245\textwidth}
         \centering  \adjustbox{valign=t}{\includegraphics[width=49mm, height=37.5mm]{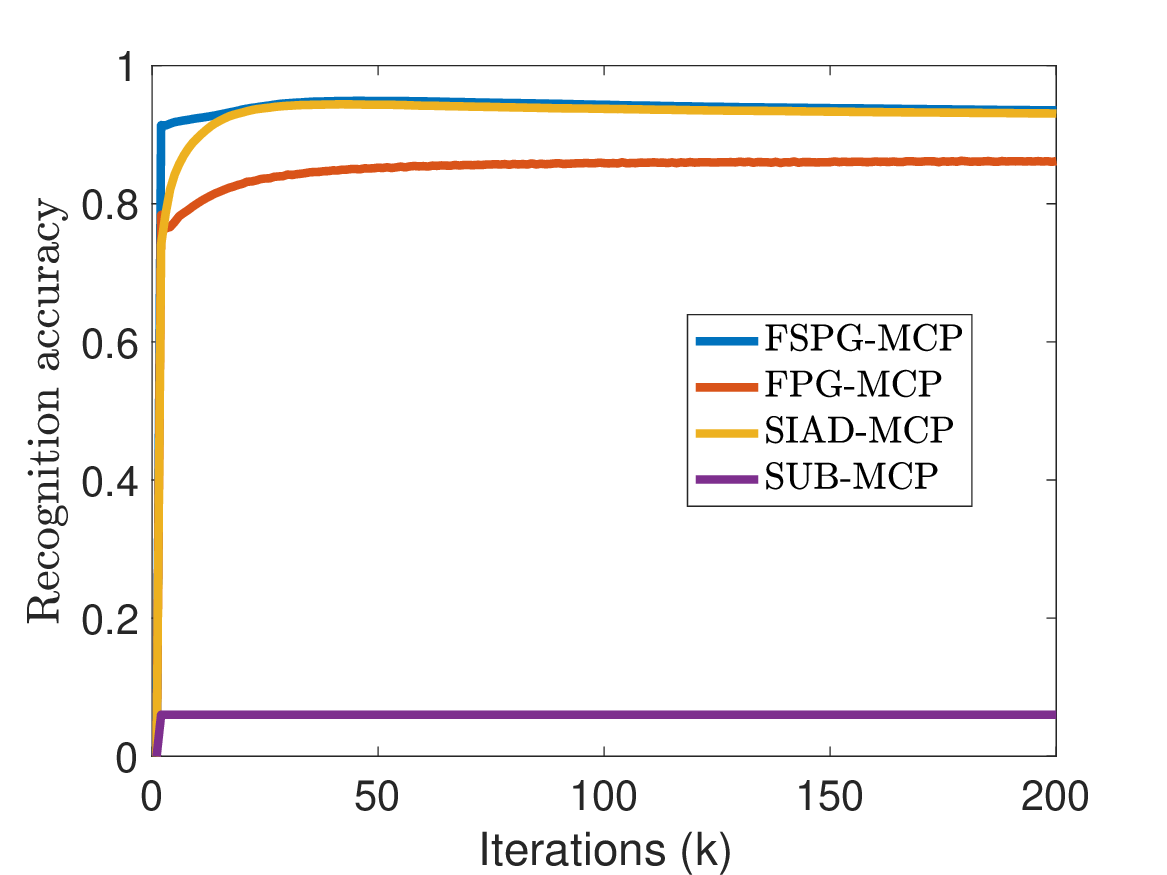}  \vspace{-1mm} }
      \caption{MCP ($\tau=0.55$)}   
     \end{subfigure}
     \hfill
     \begin{subfigure}[b]{0.245\textwidth}
         \centering  \adjustbox{valign=t}{\includegraphics[width=49mm, height=37.5mm]{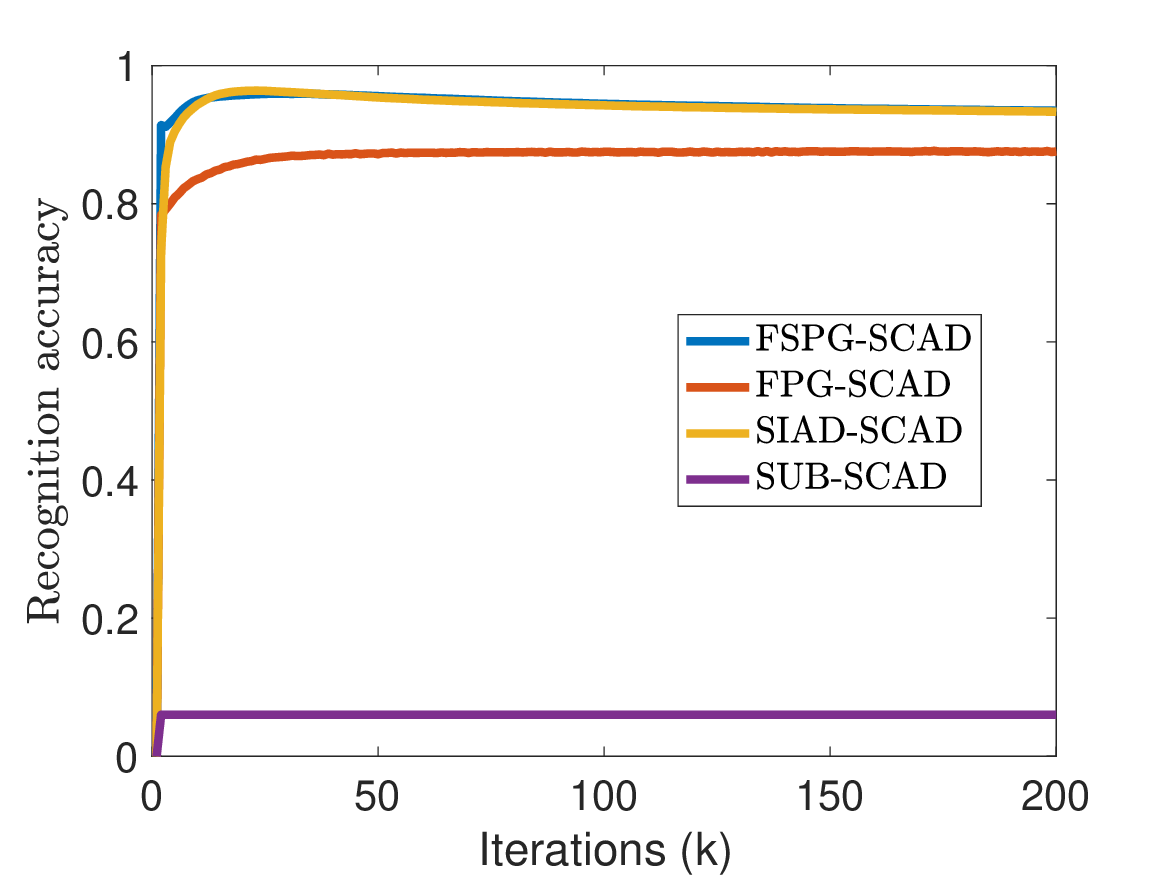}}
        \caption{SCAD ($\tau=0.55$)} 
     \end{subfigure}
     \centering
    \begin{subfigure}[b]{0.245\textwidth}
         \centering  \adjustbox{valign=t}{\includegraphics[width=49mm, height=37.5mm]{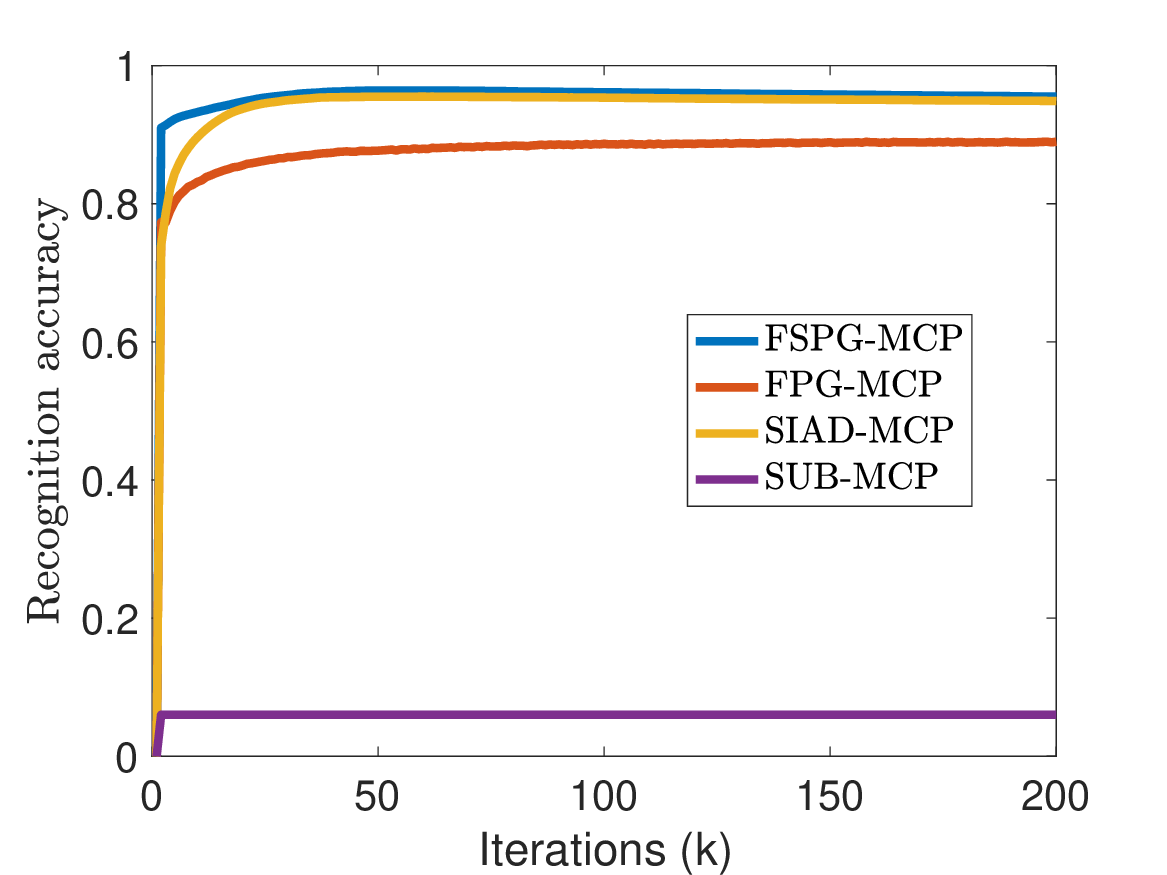}  \vspace{-1mm} }
      \caption{MCP ($\tau=0.7$)}   
     \end{subfigure}
     \hfill
     \begin{subfigure}[b]{0.245\textwidth}
         \centering  \adjustbox{valign=t}{\includegraphics[width=49mm, height=37.5mm]{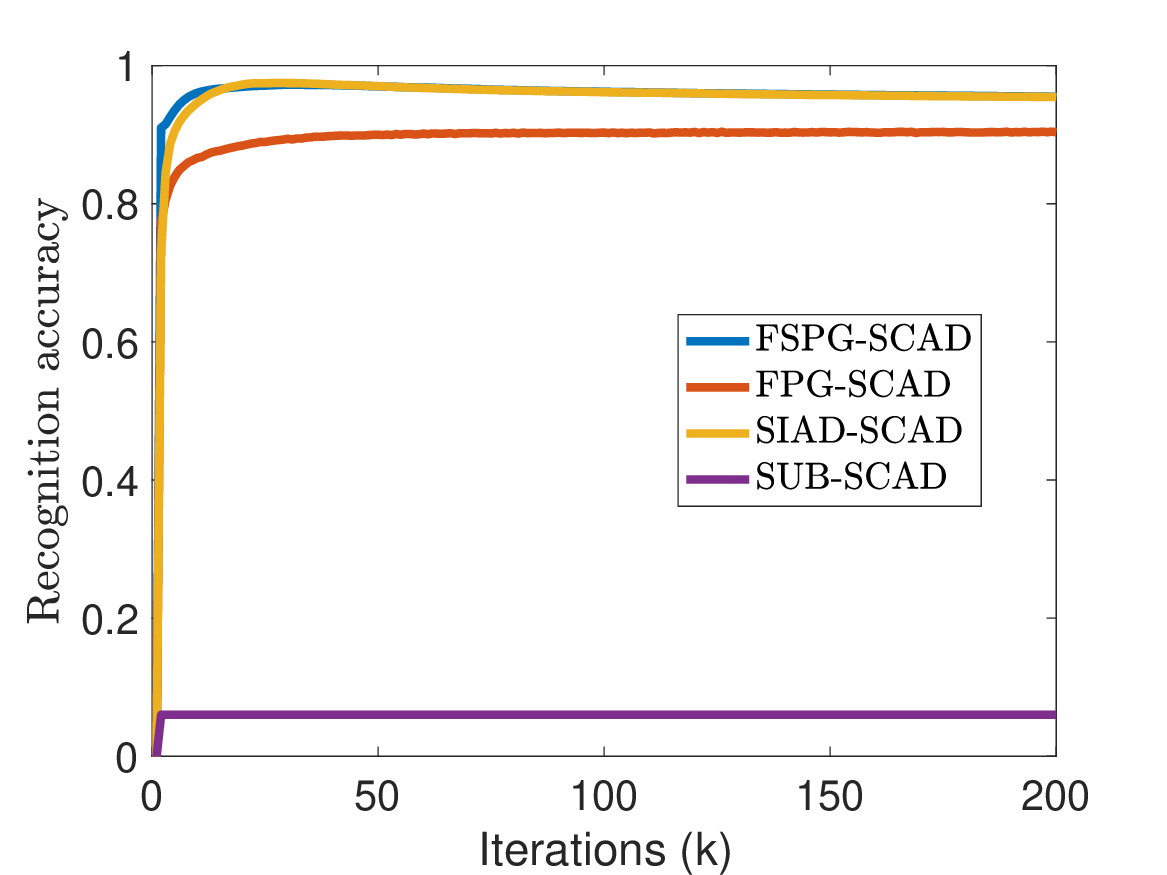}}
        \caption{SCAD ($\tau=0.7$)}
     \end{subfigure}
     \centering
   \caption{Accuracy of correctly recognizing active and non-active coefficients.}
   \label{fig2}
\end{figure*}

\begin{figure*}[ht]
\centering
\begin{subfigure}[b]{0.245\textwidth}
         \centering  \adjustbox{valign=t}{\includegraphics[width=49mm, height=37.5mm]{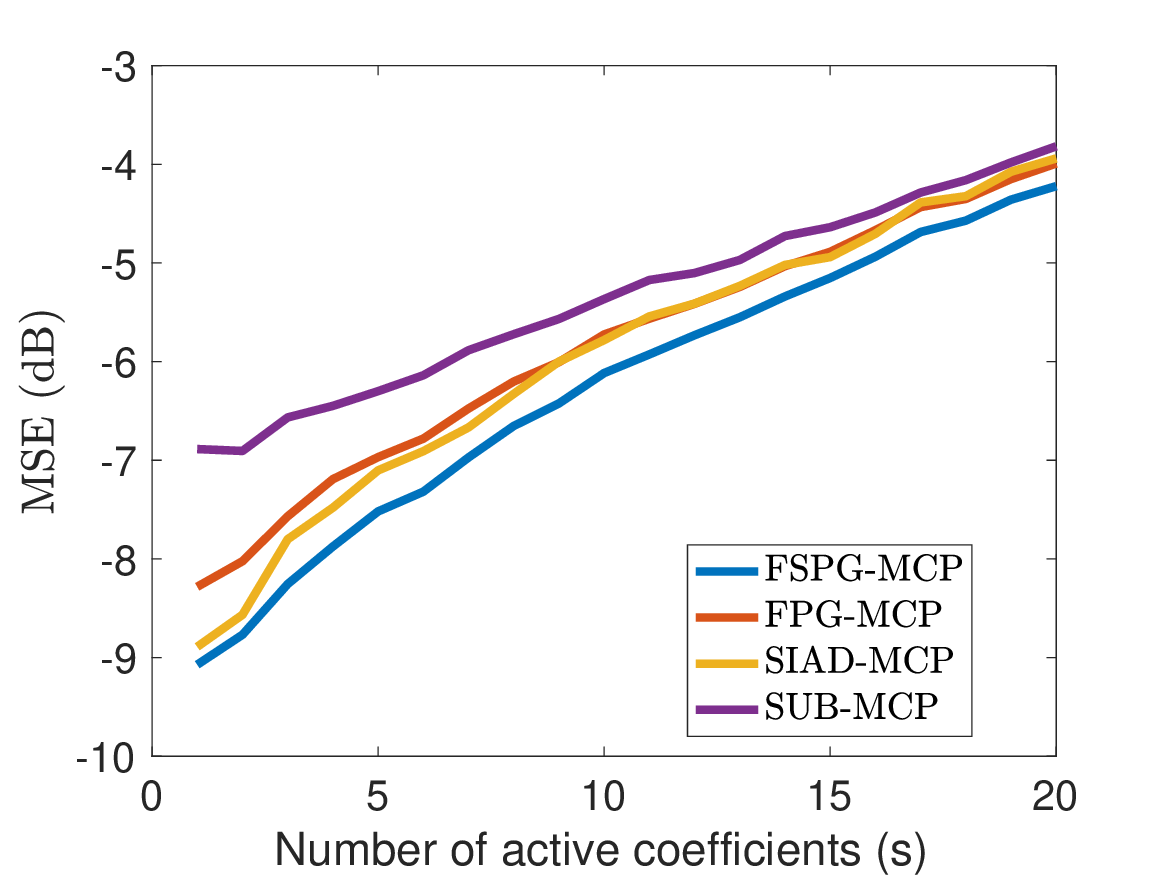}  \vspace{-1mm} }
      \caption{MCP ($\tau=0.55$)}   
     \end{subfigure}
     \hfill
     \begin{subfigure}[b]{0.245\textwidth}
         \centering  \adjustbox{valign=t}{\includegraphics[width=49mm, height=37.5mm]{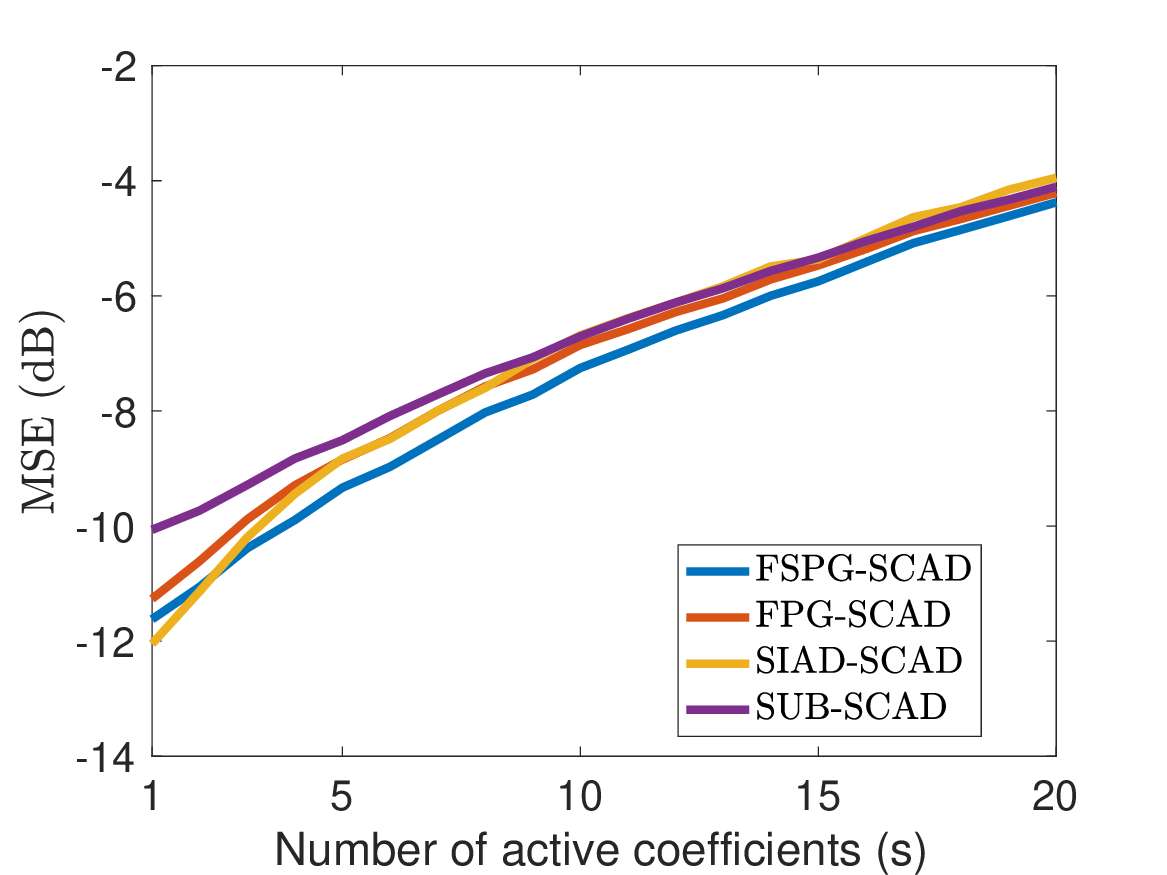}}
        \caption{SCAD ($\tau=0.55$)} 
     \end{subfigure}
     \centering
     \begin{subfigure}[b]{0.245\textwidth}
         \centering  \adjustbox{valign=t}{\includegraphics[width=49mm, height=37.5mm]{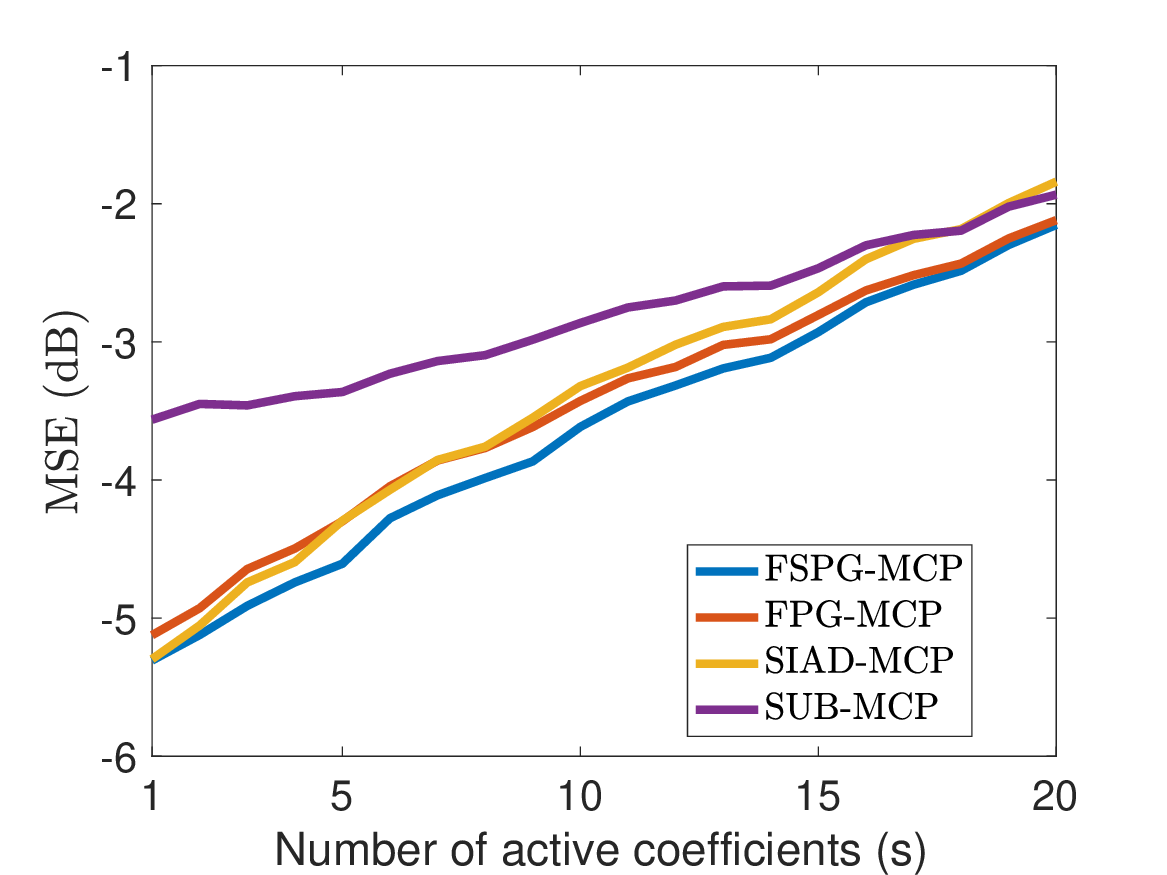}  \vspace{-1mm} }
      \caption{MCP ($\tau=0.7$)}   
     \end{subfigure}
     \hfill
     \begin{subfigure}[b]{0.245\textwidth}
         \centering  \adjustbox{valign=t}{\includegraphics[width=49mm, height=37.5mm]{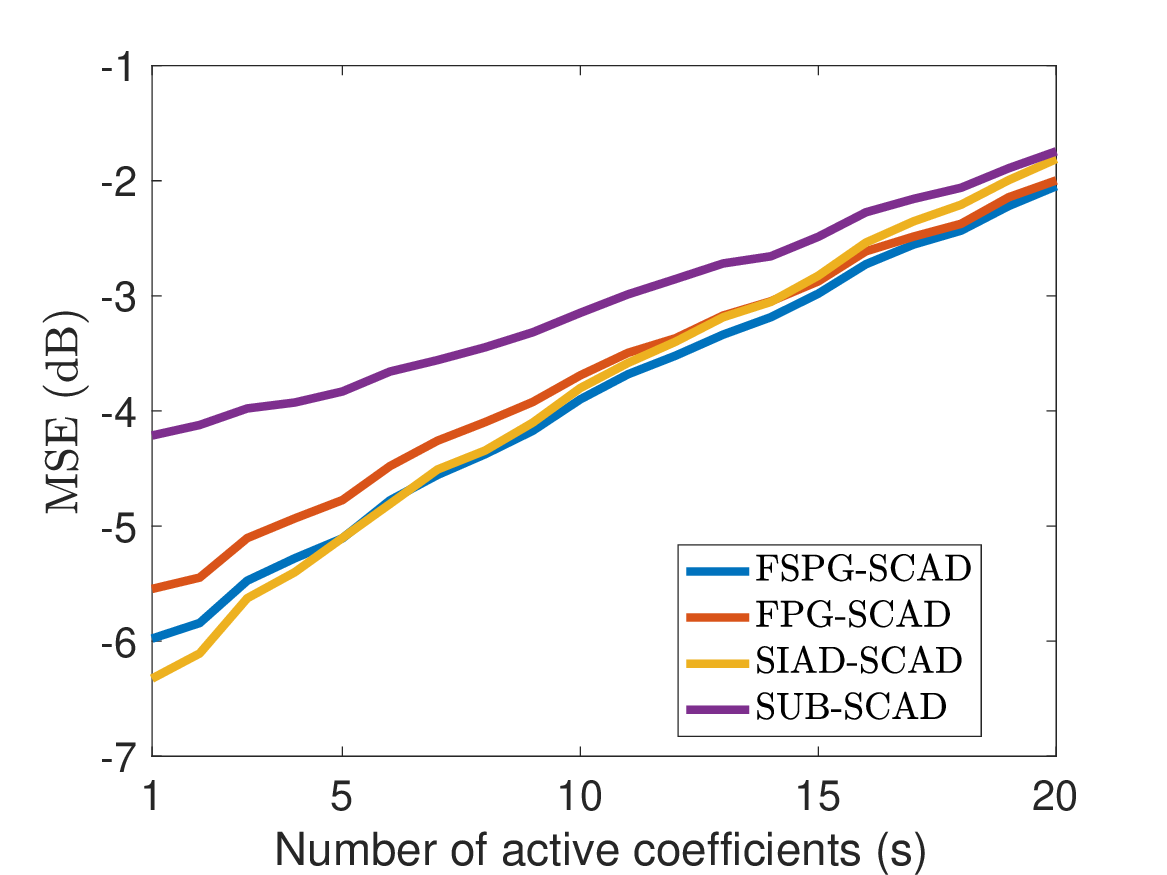}}
        \caption{SCAD ($\tau=0.7$)} 
     \end{subfigure}
     \centering
   \caption{MSE versus the number of active coefficients $s$ in model parameter $\boldsymbol{\beta}_{\tau} \in \mathbb{R}^P$.}
   \label{fig3}
\end{figure*}

\begin{figure*}[ht]
     \centering
     \begin{subfigure}[b]{0.245\textwidth}
         \centering  \adjustbox{valign=t}{\includegraphics[width=49mm, height=37.5mm]{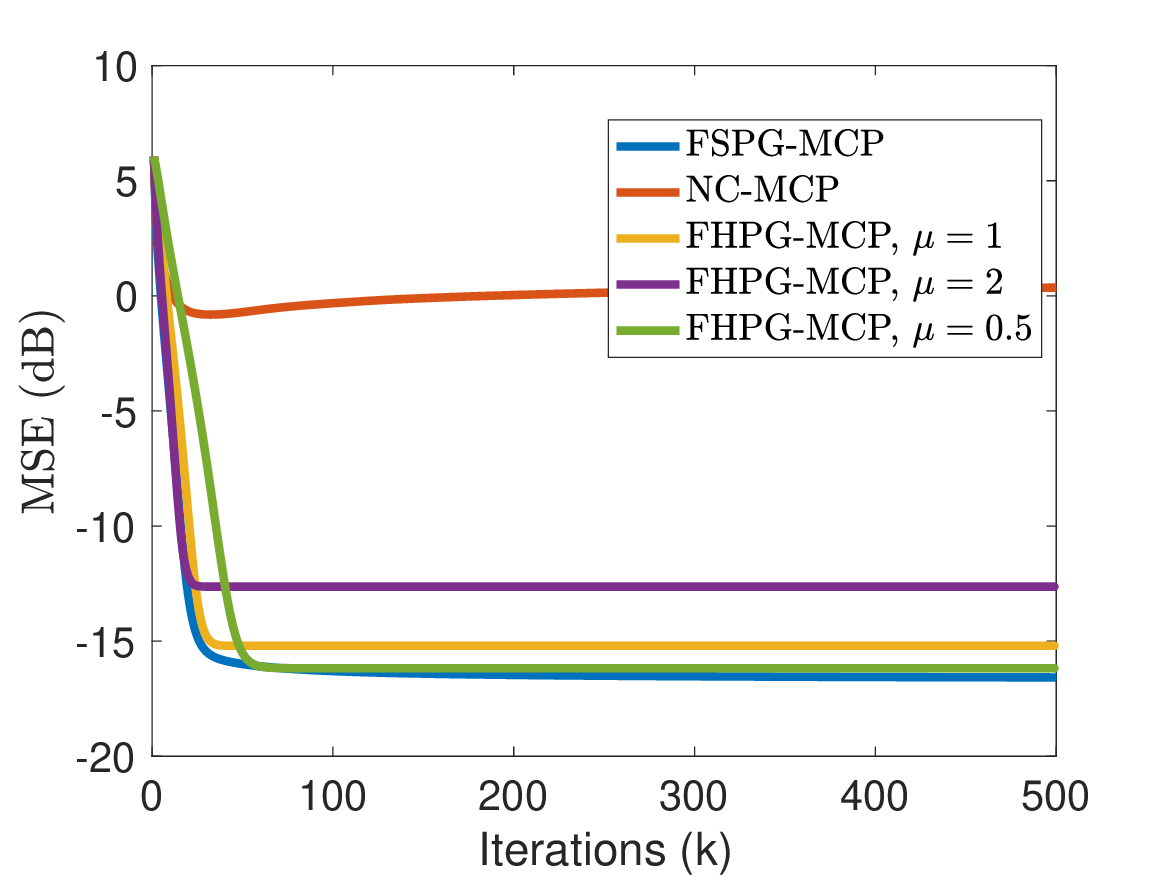}  \vspace{-1mm} }
      \caption{MCP ($\tau=0.55$)}   
     \end{subfigure}
     \hfill
     \begin{subfigure}[b]{0.245\textwidth}
         \centering  \adjustbox{valign=t}{\includegraphics[width=49mm, height=37.5mm]{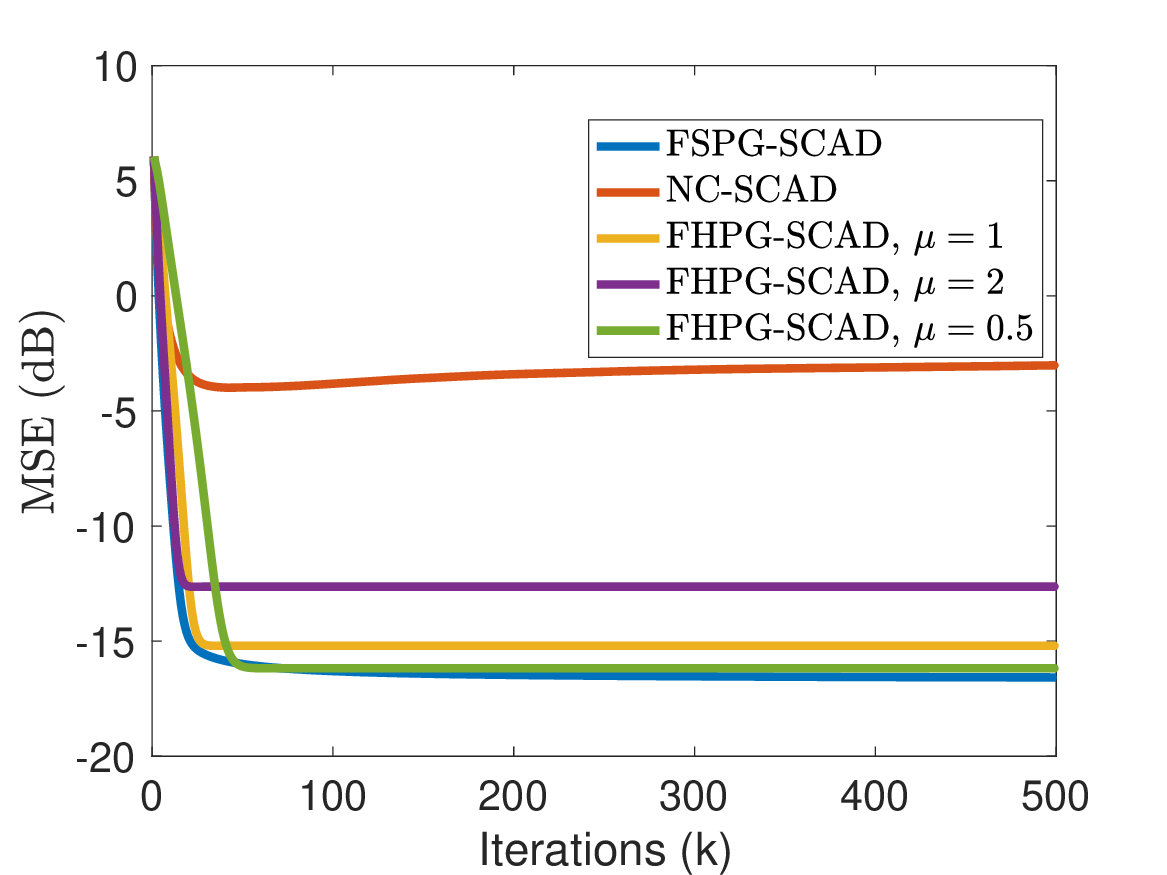}}
        \caption{SCAD ($\tau=0.55$)} 
     \end{subfigure}
    \centering
     \begin{subfigure}[b]{0.245\textwidth}
         \centering  \adjustbox{valign=t}{\includegraphics[width=49mm, height=37.5mm]{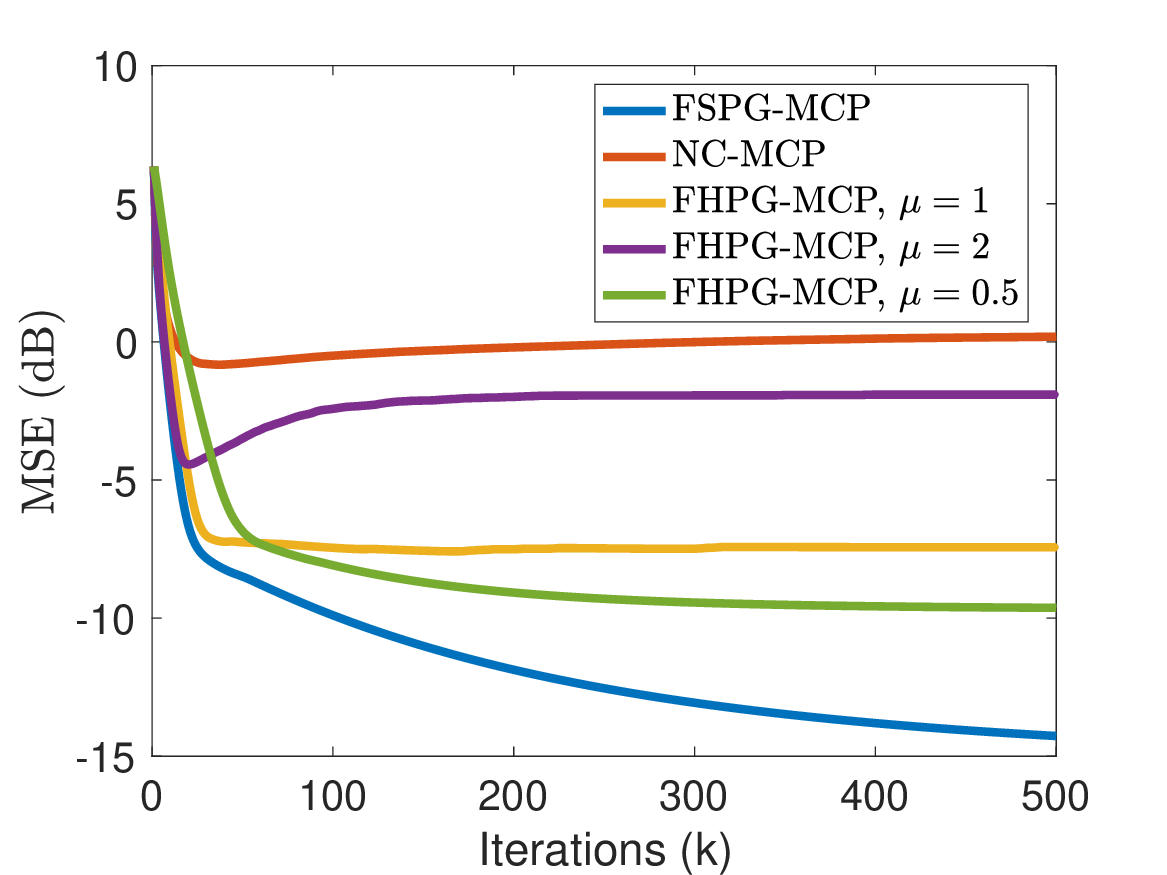}  \vspace{-1mm} }
      \caption{MCP ($\tau=0.7$)}   
     \end{subfigure}
     \hfill
     \begin{subfigure}[b]{0.245\textwidth}
         \centering  \adjustbox{valign=t}{\includegraphics[width=49mm, height=37.5mm]{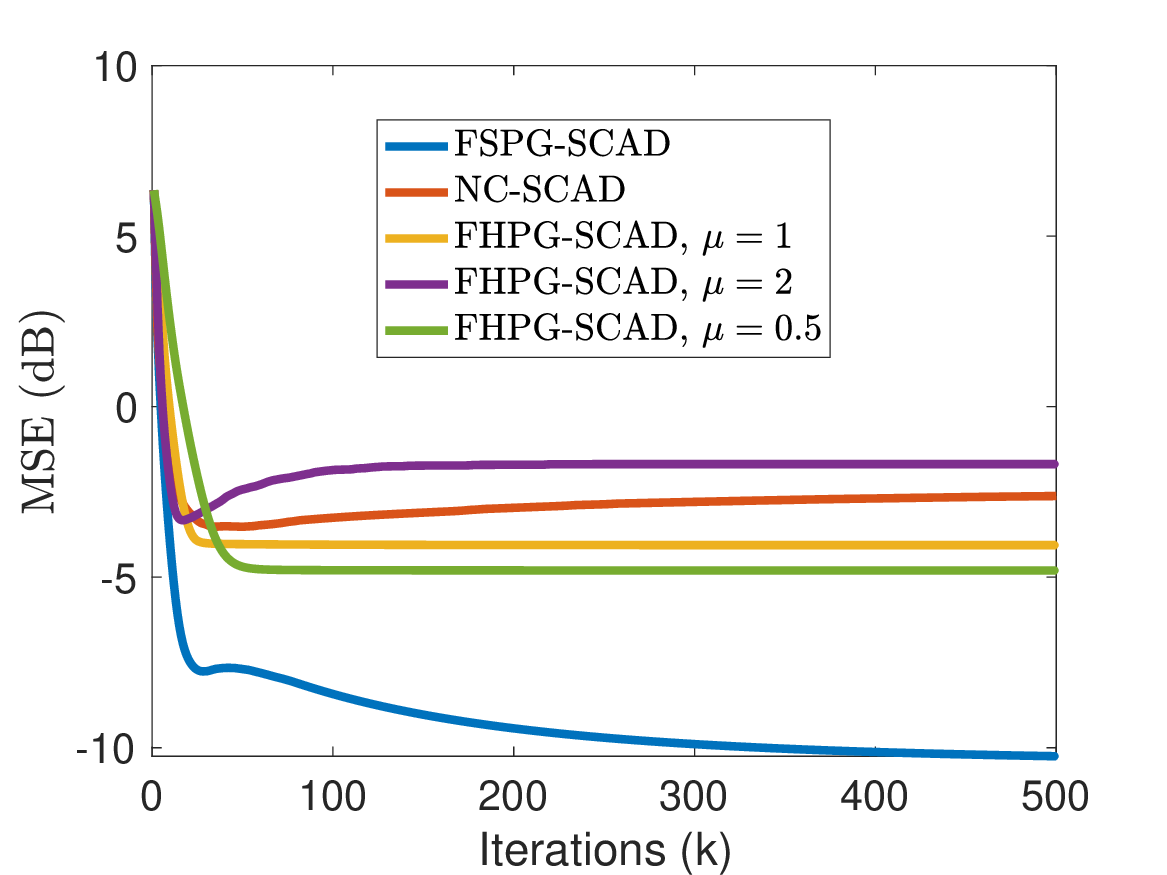}}
        \caption{SCAD ($\tau=0.7$)} 
     \end{subfigure}
     \centering
   \caption{MSE versus iterations for FSPG, FHPG, and non-cooperation scenario.}
   \label{fig4}
\end{figure*}

\begin{figure*}[ht]
     \centering
     \begin{subfigure}[b]{0.245\textwidth}
         \centering  \adjustbox{valign=t}{\includegraphics[width=49mm, height=37.5mm]{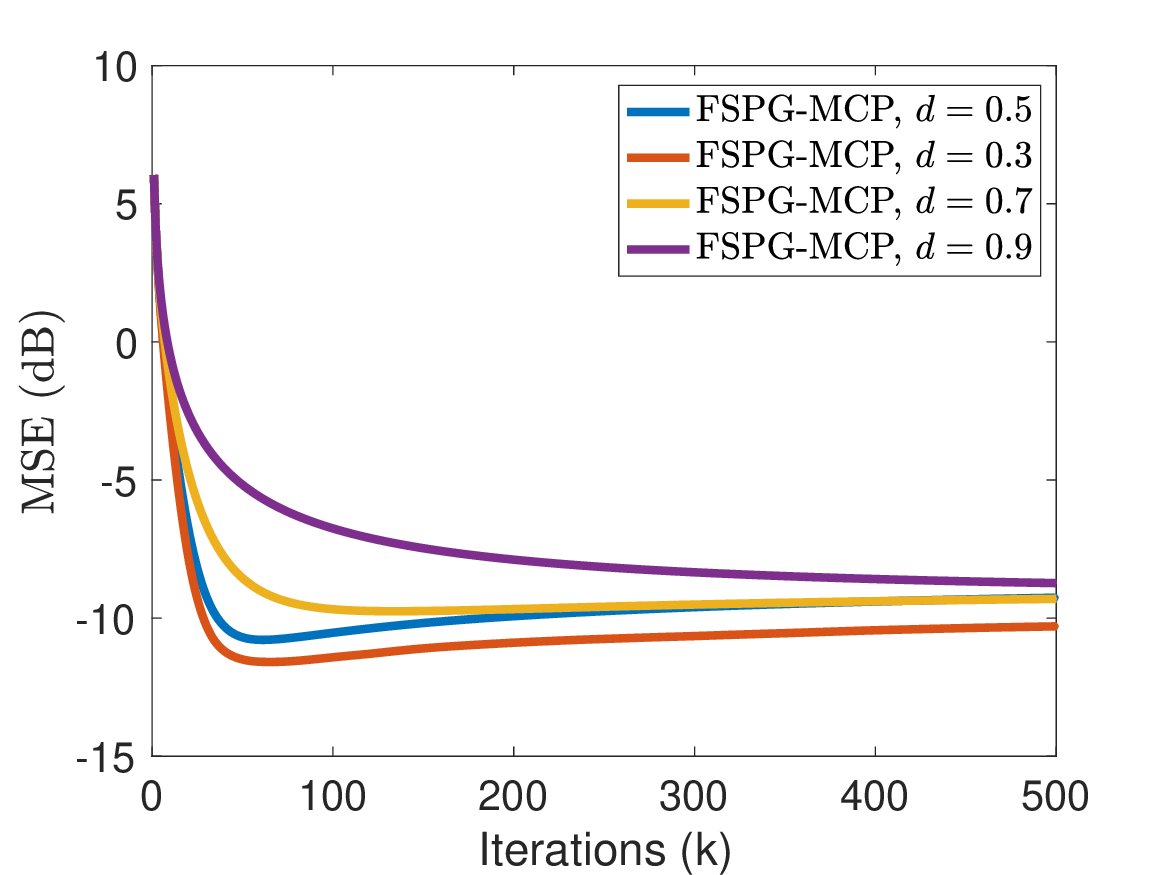}  \vspace{-1mm} }
      \caption{MCP ($\tau=0.55$)}   
     \end{subfigure}
     \hfill
     \begin{subfigure}[b]{0.245\textwidth}
         \centering  \adjustbox{valign=t}{\includegraphics[width=49mm, height=37.5mm]{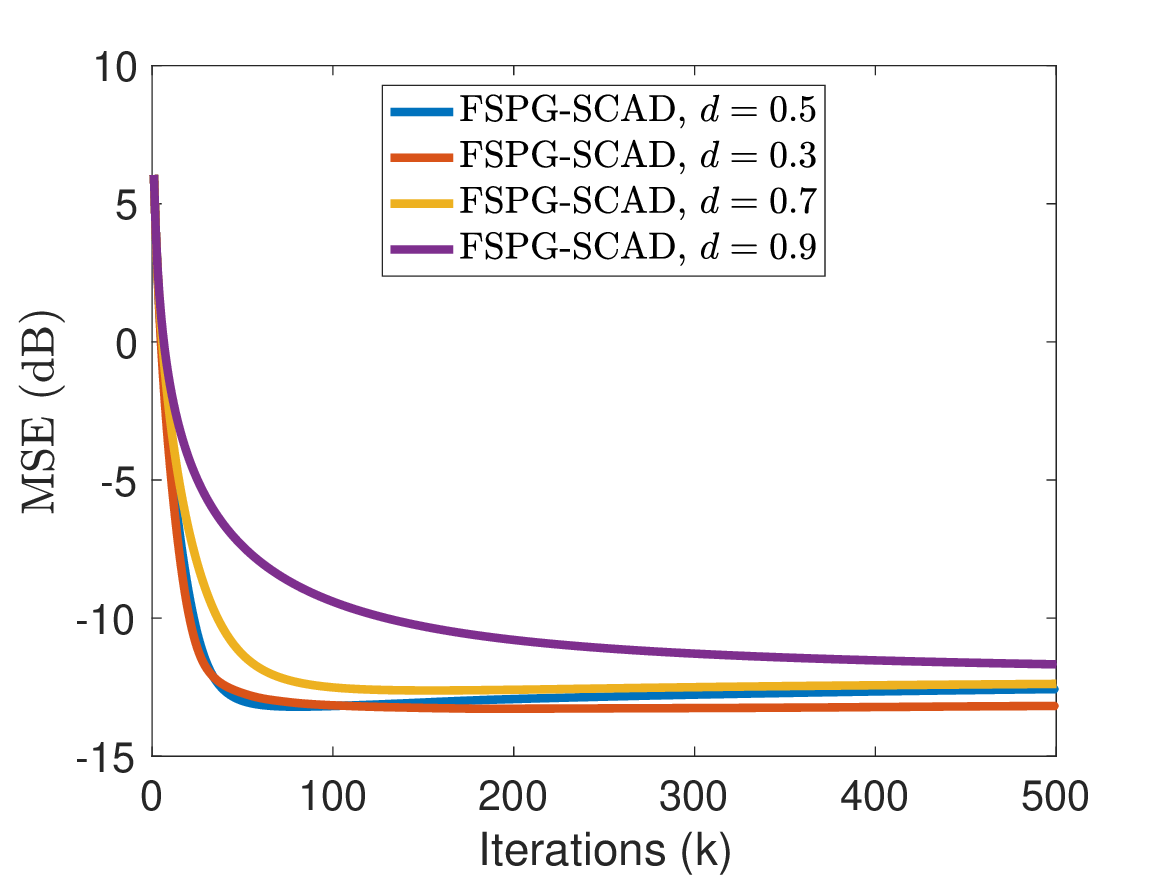}}
        \caption{SCAD ($\tau=0.55$)} 
     \end{subfigure}
    \centering
     \begin{subfigure}[b]{0.245\textwidth}
         \centering  \adjustbox{valign=t}{\includegraphics[width=49mm, height=37.5mm]{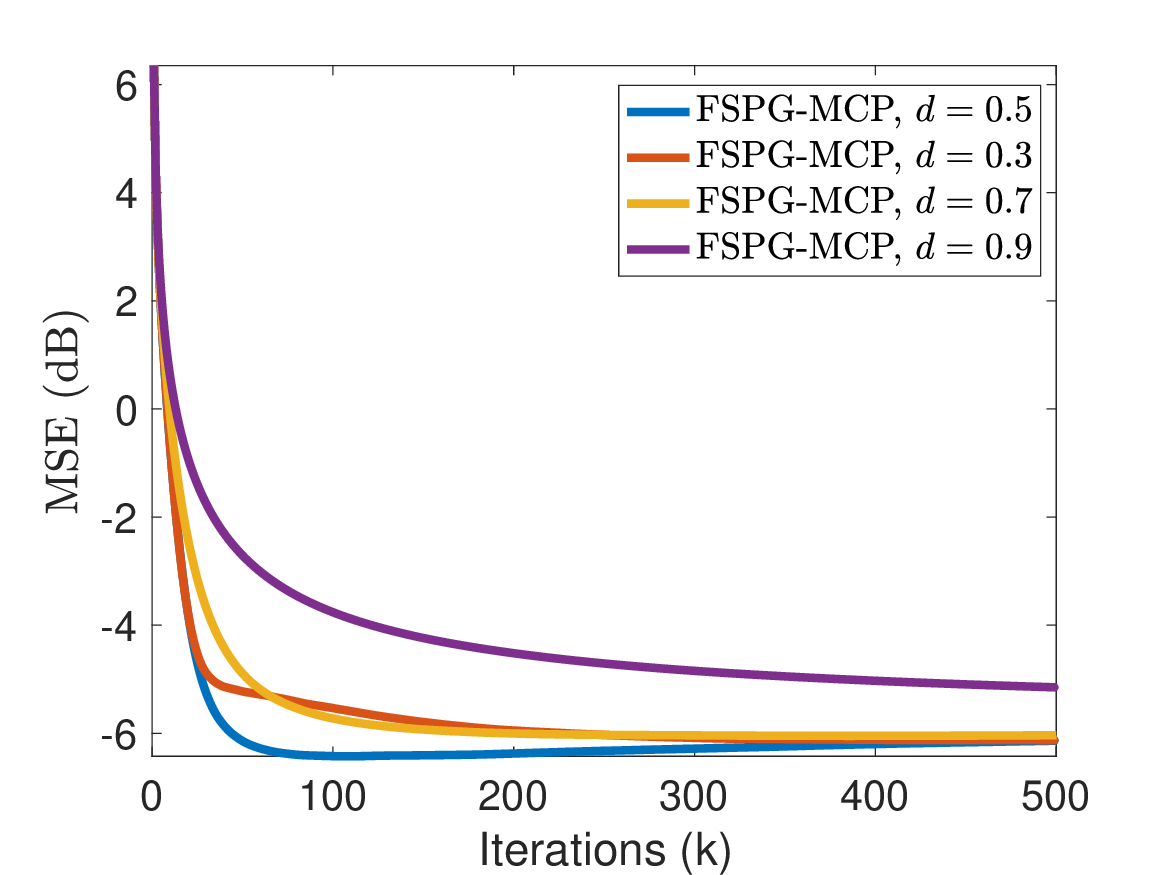}  \vspace{-1mm} }
      \caption{MCP ($\tau=0.7$)}   
     \end{subfigure}
     \hfill
     \begin{subfigure}[b]{0.245\textwidth}
         \centering  \adjustbox{valign=t}{\includegraphics[width=49mm, height=37.5mm]{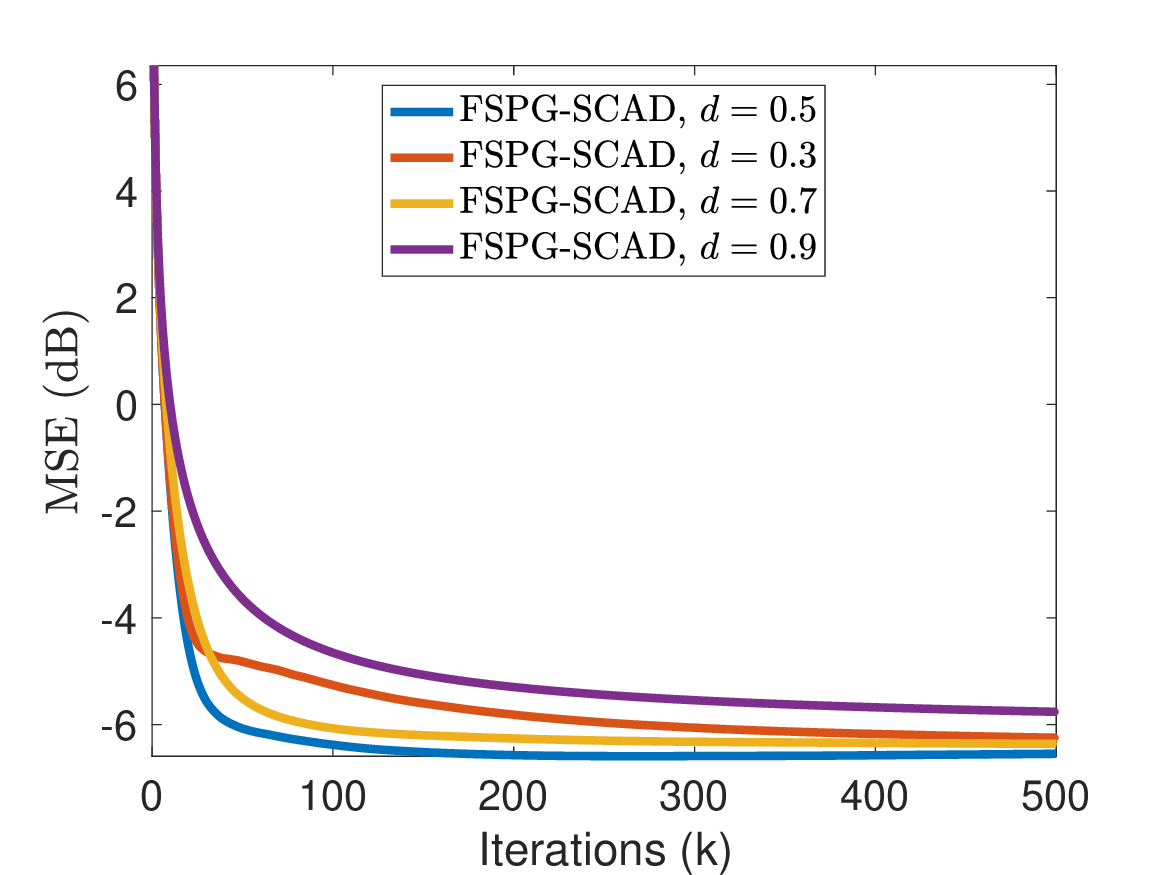}}
        \caption{SCAD ($\tau=0.7$)} 
     \end{subfigure}
     \centering
   \caption{MSE versus iterations for different $d$.}
   \label{fig5}
\end{figure*}

\begin{theorem}[Global Convergence]\label{theorem3}
Suppose there exist a constant $K$ such that for every $k \geq K$, the step size satifies $\sigma^{(k+1)} \geq n \rho$, where $\rho$ is the weak convexity parameter of $P_{\lambda,\gamma} (\cdot)$. Let the update rules for $\sigma$ and $\mu$ be given by $\sigma^{(k+1)} = c(k+1)^d$ and $\mu^{(k+1)} = \frac{\beta}{(k+1)^d}$ respectively, with $d \in (0,1)$, $c > 0$, $\beta \geq 0$, and $\beta c \geq \frac{\lambda_{\max}\left(\bar{\mathbf{X}}^\top\bar{\mathbf{X}}\right)}{2}$. Then, Algorithm \ref{alg:1} will converge to a stationary point $\mathbf{w}^{*}$ that satisfies the optimality condition:
\begin{equation}\label{eq38}
     \mathbf{0} \in \sum_{l=1}^{L} \partial {g}_l\left(\mathbf{w}^{*}\right) + n \partial P_{\lambda,\gamma}\left(\mathbf{w}^{*}\right).
\end{equation}
\end{theorem}
\begin{proof}
By Theorem \ref{theo1}, the iterative sequence $\{\mathbf{w}^{(k)}\}$ generated by Algorithm \ref{alg:1} exhibits a sufficient descent property. Specifically, the gap between consecutive iterations decreases at a rate of $o\left(k^{-1-d}\right)$, which implies that the sequence is Cauchy and thus converges to a limit point $\mathbf{w}^{*}$.

According to Theorem \ref{theo2}, there exists a sequence $\boldsymbol{\kappa}^{(k+1)}$ related to the gradients of the smoothed approximations $\tilde{g}_l$ and the sub-differential of $P_{\lambda,\gamma}$, showcasing a reduction in the norm of $\boldsymbol{\kappa}^{(k+1)}$ at the rate of $o\left(k^{-\frac{1}{2}+\frac{d}{2}}\right)$. This implies that $\boldsymbol{\kappa}^{(k+1)}$ approaches $\mathbf{0}$ as $k\rightarrow \infty$, indicating that the sequence approaches  a stationary point.

Furthermore, the update formula for $\mu^{(k+1)}$ coupled with the limiting behavior $\lim_{\mu \rightarrow 0} \nabla \tilde{g_l}(\cdot,\mu) \in \partial g_l (\cdot)$ ensure that the gradients of $\tilde{g_l}$ converge to the sub-gradients of the original $g_l$ functions at $\mathbf{w}^{*}$ as $k\rightarrow \infty$ (thus, $\mu^{(k)} \to 0$). This convergence, along with the diminishing norm of $\boldsymbol{\kappa}^{(k+1)}$, confirms that $\mathbf{w}^{*}$ satisfies the optimality condition given by \eqref{eq38}. Therefore, Algorithm \ref{alg:1} converges to a stationary point of the original problem, as required.
\end{proof}

\begin{figure*}[ht]
     \centering
     \begin{subfigure}[b]{0.245\textwidth}
         \centering  \adjustbox{valign=t}{\includegraphics[width=49mm, height=37.5mm]{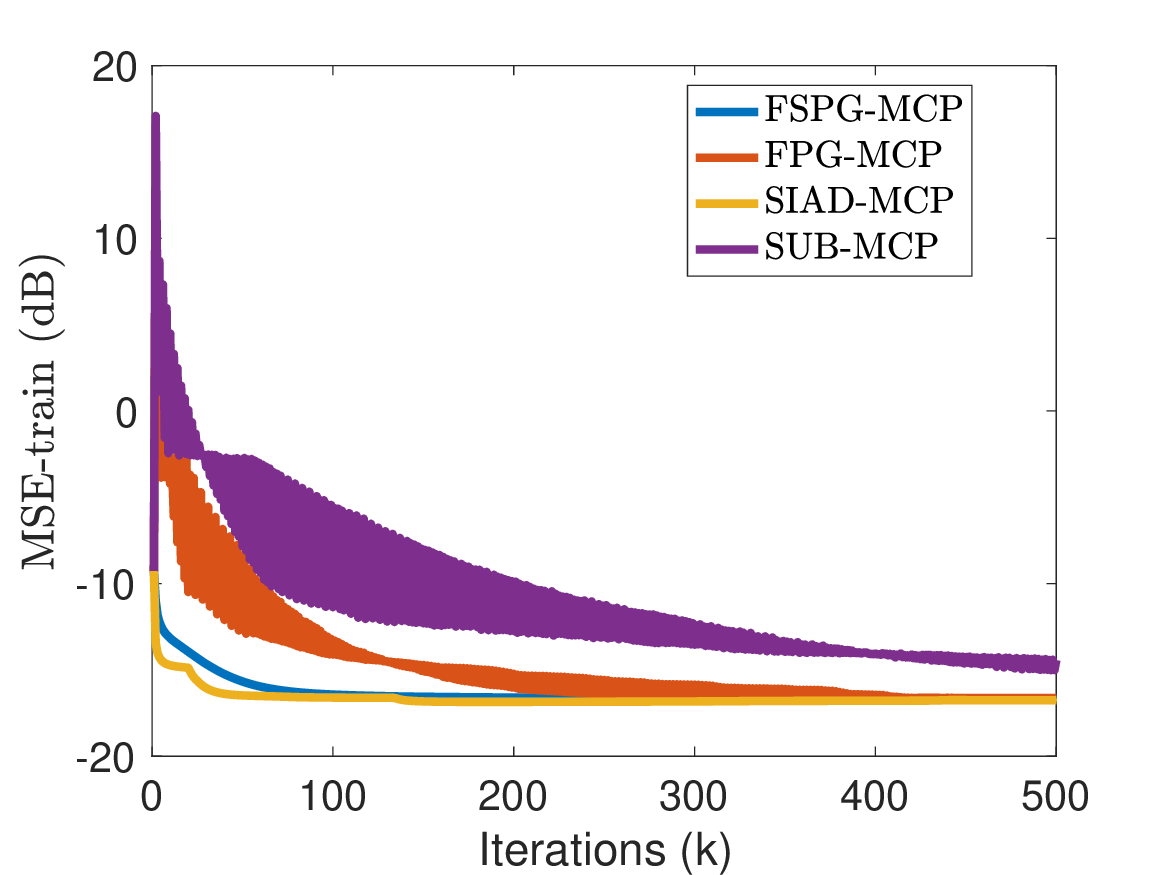}  \vspace{-1mm} }
      \caption{MCP (Train)}   
     \end{subfigure}
     \hfill
     \begin{subfigure}[b]{0.245\textwidth}
         \centering  \adjustbox{valign=t}{\includegraphics[width=49mm, height=37.5mm]{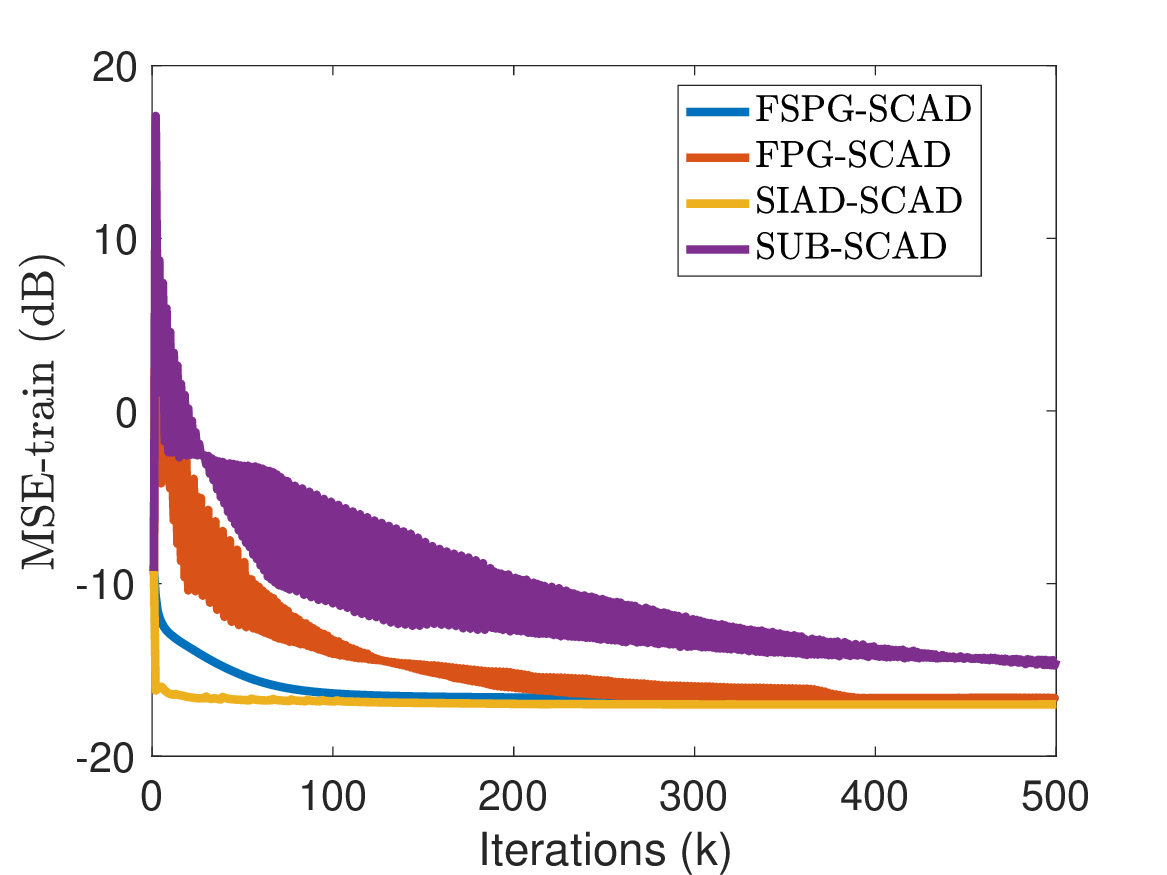}}
        \caption{SCAD (Train)} 
     \end{subfigure}
    \centering
     \begin{subfigure}[b]{0.245\textwidth}
         \centering  \adjustbox{valign=t}{\includegraphics[width=49mm, height=37.5mm]{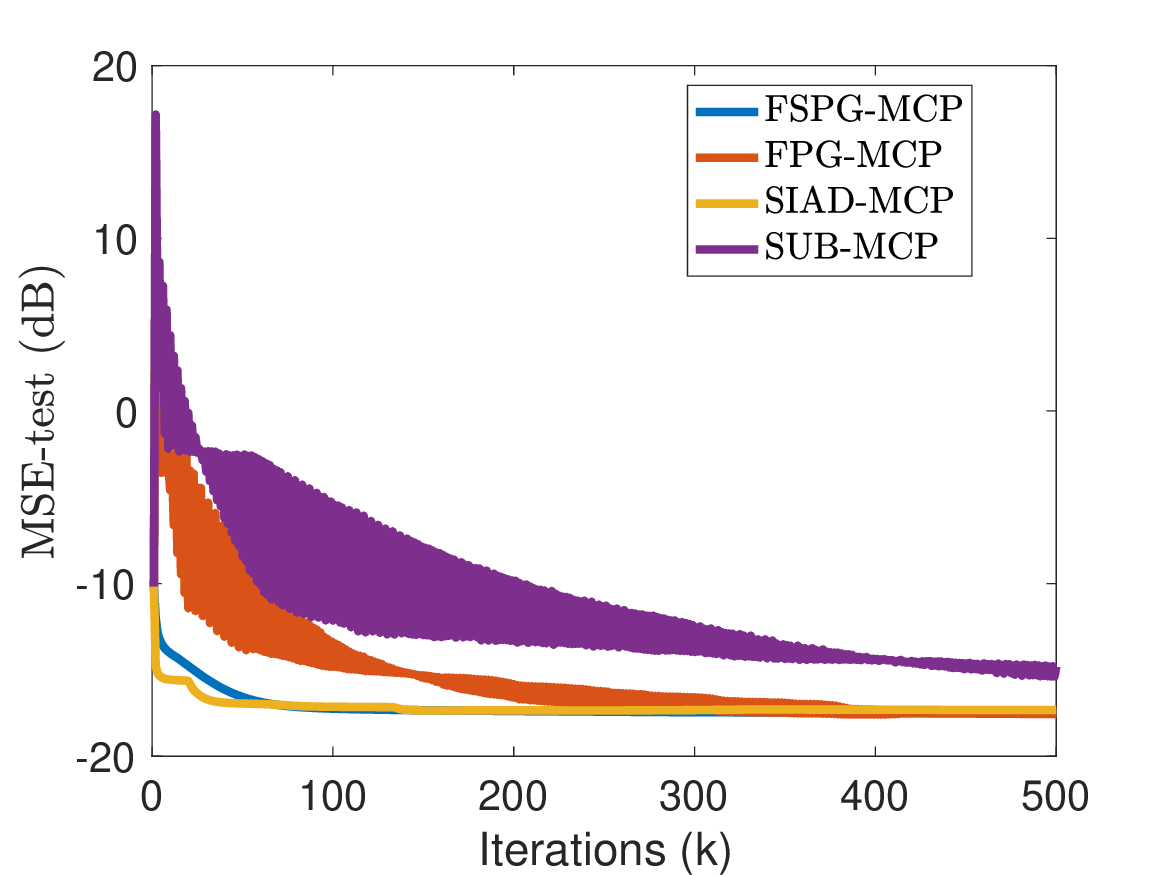}  \vspace{-1mm} }
      \caption{MCP (Test)}   
     \end{subfigure}
     \hfill
     \begin{subfigure}[b]{0.245\textwidth}
         \centering  \adjustbox{valign=t}{\includegraphics[width=49mm, height=37.5mm]{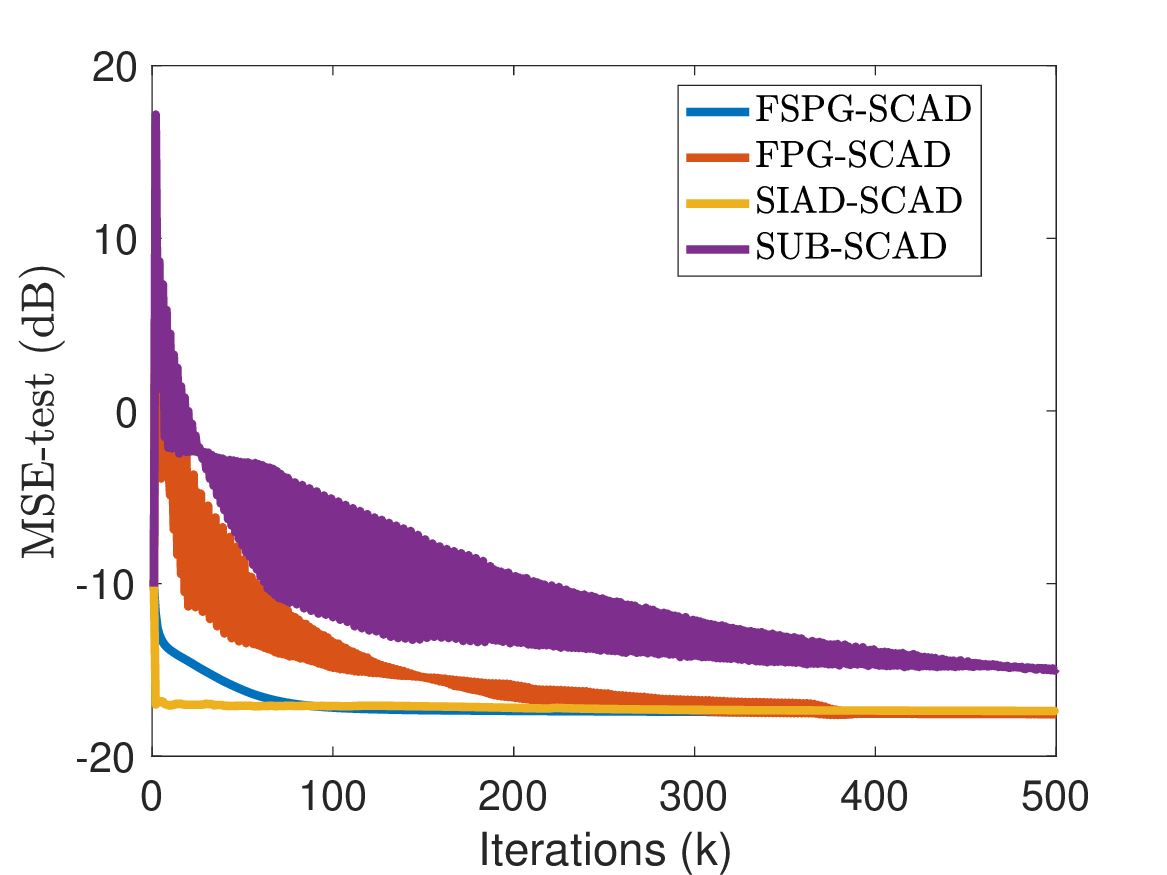}}
        \caption{SCAD (Test)} 
     \end{subfigure}
     \centering
   \caption{}
   \label{fig6}
\end{figure*}

\section{Simulation Results}\label{sec5}
In this section, we present an extensive simulation study to evaluate the efficacy of our federated smoothing proximal gradient (FSPG) algorithm in the context of sparse quantile regression. We compare the FSPG with leading contemporary approaches, including the proximal gradient descent with increasing penalty (PGD) \cite{gao2024stochastic} and the sub-gradient method (SUB) \cite{mirzaeifard2023distributed} tailored for federated settings, and the smoothing time-increasing penalty ADMM (SIAD) \cite{mirzaeifard2023smoothing} as a benchmark for centralized scenarios. Our evaluation criteria include convergence rate, efficiency in minimizing the mean square error (MSE) across both synthetic and real-world datasets, as well as accuracy in recognizing active and non-active coefficients. 

Additionally, we compare the performance and convergence dynamics of the FSPG algorithm against the Federated Huber Loss Proximal Gradient (FHPG) –- a variant of FSPG with static $\mu$ and $\sigma_{\lambda}$ parameters –- focusing on the MSE performance. This examination aims to highlight the nuanced advantages and potential trade-offs inherent in using the FSPG approach to quantile regression.

\subsection{Simulation Setup}
We evaluate the performance of the FSPG algorithm across six distinct simulation scenarios. In these evaluations, the penalty parameters are set to $\gamma_{\text{SCAD}} = 3.1$, $\gamma_{\text{MCP}} = 2.4$, $\beta = 4$, and $c = \frac{\lambda_{\max}\left(\bar{\mathbf{X}}^\top\bar{\mathbf{X}}\right)}{8}$. 

For all synthetic data scenarios, we fix $\lambda=0.055$, and the results are averaged over $1000$ independent experiments. The primary performance metric for synthetic data is the MSE, defined as $\text{E}\left\|\mathbf{\hat{w}} - \mathbf{w}\right\|_2^2$. We assess algorithm performance based on convergence speed, MSE, and accuracy in identifying active and inactive coefficients. 

For real-world datasets, we utilize a $\lambda=0.01$ and evaluate MSE performance for both training and testing datasets, calculated as $\frac{1}{n} \sum_{i=1}^{n} \left\|y_i - \bar{\mathbf{x}}_i^{\top}\mathbf{w}\right\|_2^2$.

The synthetic data generation process is as follows: For each observation vector $\left(\tilde{x}^{(l)}_{i,1}, \tilde{x}^{(l)}_{i,2}, \ldots, \tilde{x}^{(l)}_{i,P}\right)^{\top}$, drawn i.i.d samples from a multivariate normal distribution $\mathcal{N}\left(\mathbf{0}, \boldsymbol{\Sigma}\right)$ for all $i \in \{1, \ldots, M_l\}$ and $l \in \{1, \ldots, L\}$, where the covariance between any two elements 
$p,q$ is defined by ${\Sigma}{p,q} = 0.5^{|p-q|}$. We apply a transformation $x^{(l)}_{i,1} = \Phi\mathopen{}\left(\tilde{x}^{(l)}_{i,1}\right)\mathclose{}$ to the first component of each observation vector, where $\Phi(\cdot)$ denotes the cumulative distribution function of the standard normal distribution. The remaining components $x^{(l)}_{i,p} = \tilde{x}_{i,p}$ for $p = 2, 3, \ldots, P$ are kept unaltered. Furthermore, we augment each observation matrix $\mathbf{X}^{(l)}$ for all $l \in \{1, \ldots, L\}$ with a column of ones to form $\tilde{\mathbf{X}}^{(l)} = [(\mathbf{X}^{(l)})^\top, \mathbf{1}]^\top$.

For scenarios one, two, four, and five the measurement model for generating data is given by:

\begin{equation*}
    y^{(l)}_i = \sum_{p=1}^{P} \xi_p x^{(l)}_{i,p} + x^{(l)}_{i,6} + x^{(l)}_{i,12} + x^{(l)}_{i,15} + x^{(l)}_{i,20} + 0.7 \epsilon^{(l)}_i x^{(l)}_{i,1} + \nu^{(l)}_i,
\end{equation*}
where $\epsilon^{(l)}_i \distas{\text{i.i.d}}  \mathcal{N}(0,1)$, $\xi_p \distas{\text{i.i.d}} \mathcal{N}\left(0,10^{-6}\right)$, and $\nu^{(l)}_i \distas{\text{i.i.d}} \mathcal{N}\left(0,(\Upsilon^{(l)})^{2}\right)$, in which $\Upsilon^{(l)})=\left|\upsilon^{(l)}\right|$ where each $\upsilon^{(l)} \distas{\text{i.i.d}} \mathcal{N}(0,1)$. Under these settings, the model is characterized as a compressible system \cite{lima2014sparsity}. The $\tau$-th conditional quantile linear function can be formulated as $\sum_{p=1}^{P} \xi_p x_{i,p} + x_{i,6} + x_{i,12} + x_{i,15} + x_{i,20} + 0.7 \Phi^{-1}(\tau) x_{i,1}+Q^{\nu}_{\tau}$ in which $Q^{\nu}_{\tau}$ can be obtained by:
\begin{equation}\label{eq:quantile_error}
   Q^{\nu}_{\tau}=\left\{ Q:\int_{-\infty}^{Q} \frac{1}{L}\sum_{l=1}^{L}\frac{1}{\Upsilon^{(l)}\sqrt{2\pi} } e^{-\frac{1}{2}\left(\frac{x}{\Upsilon^{(l)}}\right)^2}dx=\tau\right\}.
\end{equation}
We assume   $\mathopen{}\left(M, L, P\right)\mathclose{}=\mathopen{}\left(10,20,100\right)\mathclose{}$ for the first, second, and fifth scenarios, where $M_l=M$ for all $l \in \{1,\cdots, L\}$. For the fourth scenario to have more data in local data in each agent, we assume $\mathopen{}\left(M, L, P\right)\mathclose{}=\mathopen{}\left(100,10,100\right)\mathclose{}$, where $M_l=M$ for all $l \in \{1,\cdots, L\}$. In all of these scenarios, we examine the results for two different values of $\tau$ as  $0.55$ and $0.7$. 

For the third scenario, we adjust the measurement model for generating data as follows:
\begin{align*}
 y^{(l)}_i = \sum_{p=1}^{P} \xi_p x^{(l)}_{i,p} + \sum_{p \in \mathcal{M}} x^{(l)}_{i,p}  + 0.7 \epsilon^{(l)}_i x^{(l)}_{i,1} + \nu^{(l)}_i,
\end{align*}
with $\mathcal{M}\subseteq \{2,\ldots,P\}$, where $\epsilon^{(l)}_i \distas{\text{i.i.d}}  \mathcal{N}(0,1)$, $\xi_p \distas{\text{i.i.d}} \mathcal{N}\left(0,10^{-6}\right)$, and $\nu^{(l)}_i \distas{\text{i.i.d}} \mathcal{N}\left(0,(\Upsilon^{(l)})^{2}\right)$, in which $\Upsilon^{(l)})=|\upsilon^{(l)}|$ where each $\upsilon^{(l)} \distas{\text{i.i.d}} \mathcal{N}(0,1)$. In this scenario, the model also presents a compressible system \cite{lima2014sparsity}, but with a different structure compared to the other scenarios with synthetic data. The $\tau$-th conditional quantile linear function can be characterized by $\sum_{p=1}^{P} \xi_p x_{i,p} + \sum_{p \in \mathcal{M}} x_{i,p} + 0.7 \Phi^{-1}(\tau) x_{i,1}+Q^{\nu}_{\tau}$, where $Q^{\nu}_{\tau}$ can also be derived by \eqref{eq:quantile_error}. we take $\tau \in \{0.55,0.7\}$ and $\mathopen{}\left(M, L, P\right)\mathclose{}=\mathopen{}\left(10,10,20\right)\mathclose{}$, where $M_l=M$ for all $l \in \{1,\cdots, L\}$.

In the sixth scenario, we evaluate the efficacy of the FSPG algorithm on the Communities and Crime dataset, sourced from the UCI Machine Learning Repository \cite{misc_communities_and_crime_unnormalized_211}. We assume a linear relationship between demographic variables (comprising 99 features) and the metric ViolentCrimesPerPop, which denotes the total number of violent crimes per $100,000$ population. Our objective is to fit a linear model to these data and pinpoint demographic features that significantly impact the crime rate. The dataset is randomly partitioned across $10$ clients, with $80\%$ allocated to the training set and the remaining $20\%$ designated as the test set.
\subsection{Results}
In the first scenario, we assessed the algorithms based on their convergence speed and MSE. The learning curves, which plot MSE against the number of iterations for different values of \(\tau = \{0.55, 0.7\}\), are shown in Figure \ref{fig1}. We see that the FSPG algorithm consistently outperforms existing methods in reducing MSE, regardless of the chosen penalty function and the specific value of \(\tau\). Moreover, the FSPG algorithm demonstrates a faster convergence rate compared to the FPG and SUB algorithms and achieves convergence speeds comparable to those of the SIAD algorithm.

In the second scenario, we evaluated the algorithms based on their ability to accurately identify active and non-active coefficients. The accuracy metric is defined as the proportion of correctly identified active and non-active coefficients relative to the total number of coefficients. Figure \ref{fig2} shows the accuracy as a function of iteration number for each algorithm. We see that the FSPG algorithm matches the accuracy of the SIAD algorithm, demonstrating that FSPG performs comparably to state-of-the-art methods in a centralized setting. In contrast, the SUB algorithm demonstrates significantly inferior performance, primarily due to its inability to produce exact zeros, thus necessitating a post-hoc thresholding of coefficients. Meanwhile, the FPG algorithm also fails to achieve results on par with those of FSPG.

In the third scenario, we examined the robustness of the algorithms across varying sparsity levels. The performance of each algorithm was assessed by tracking the MSE after $10,000$ iterations as the number of active coefficients was incrementally increased from $1$ to $P$. Figure \ref{fig3} illustrates the relationship between MSE and the number of active coefficients for each algorithm. As depicted in Figure \ref{fig3}, the FSPG algorithm demonstrates consistent performance across the entire range of sparsity, from highly sparse to dense configurations. In comparison, other state-of-the-art algorithms generally show poorer performance at most sparsity levels.

In the fourth scenario, we evaluated the convergence properties of the FSPG algorithm and compared them with those of the FHPG algorithm and the non-cooperation quantile (NC) algorithm, where each client estimates the parameters independently. For FHPG, we tested three different smoothing parameters, $\mu = \{1, 2, 0.5\}$, and adjusted penalty parameters to ensure convergence. Figure \ref{fig4} displays the learning curves for these algorithms. As can be seen, despite each client having $100$ samples, the NC algorithm exhibits poor MSE performance. While FHPG successfully converges to a stationary point, the performance of the FSPG marginally surpasses that of FHPG, particularly with the tested smoothing parameters. Specifically, FSPG achieves the rapid convergence rates of FHPG when high smoothness parameters are used, and maintains the high accuracy similar to FHPG with low smoothness parameters.

In the fifth scenario, we explored the impact of the parameter \(d\) on the convergence behavior of the FSPG algorithm. This parameter governs the frequency of updating the parameters \(\mu\) and \(\sigma\). We tested values of \(d = \{0.3, 0.5, 0.7, 0.9\}\). Figure \ref{fig5} shows the learning curves of FSPG for these settings. The results indicate that a setting of \(d = 0.5\) leads to faster convergence. In contrast, higher values such as \(d = 0.7\) and \(d = 0.9\) yield more robust but slower convergence patterns.

In the final scenario, we assessed the performance of the FSPG algorithm on a real-world dataset. Figure \ref{fig6} illustrates the relationship between the MSE of the training and testing datasets and the number of iterations for each algorithm when $\tau$ is fixed to $0.5$ for robust regression. The learning curves reveal that while the FPG and SUB algorithms exhibit significant fluctuations, the FSPG and SIAD algorithms demonstrate a consistent decrease in MSE with increasing iterations. This stability is attributed to the ability of FSPG and SIAD to systematically reduce the objective function in each iteration, a property not guaranteed by the FPG and SUB algorithms.

\section{Conclusion}\label{sec6}
This paper introduced the federated smoothing proximal gradient (FSPG) approach, a novel approach for handling sparse-penalized quantile regression with non-convex and non-smooth penalties in federated settings. We rigorously demonstrated the convergence behavior of the FSPG, highlighting its efficiency with a sub-gradient bound rate of $o\left(k^{-\frac{1}{2}+\frac{d}{2}}\right)$ for the proximal gradient steps. This represents a significant enhancement over traditional methods such as the sub-gradient (SUB) algorithm and proximal gradient descent with time-increasing penalty (PGD). Uniquely, FSPG guarantees quantifiable improvements in model accuracy with each iteration, a feature not consistently provided by other methods.  The algorithm achieves a reduction in the fluctuation of coefficients across successive iterations at a rate of $o\left(k^{-\frac{1}{2}-\frac{d}{2}}\right)$, markedly superior to conventional methods. This advantage underscores the superior optimization dynamics of our approach, especially suited for the federated settings. Empirical results confirm that FSPG is more capable of providing accurate coefficients in complex scenarios, and excels in precisely identifying active and non-active coefficients. This capability ensures more reliable and interpretable models, which are crucial for practical applications requiring precise predictive analytics. As a result, the FSPG algorithm emerges as a robust and efficient framework for addressing the complex challenges posed by non-smooth and non-convex objective functions in federated quantile regression. This method significantly advances the field by enhancing the reliability and interpretability of the models in decentralized data environments.



\begin{thebibliography}{10}
\providecommand{\url}[1]{#1}
\csname url@samestyle\endcsname
\providecommand{\newblock}{\relax}
\providecommand{\bibinfo}[2]{#2}
\providecommand{\BIBentrySTDinterwordspacing}{\spaceskip=0pt\relax}
\providecommand{\BIBentryALTinterwordstretchfactor}{4}
\providecommand{\BIBentryALTinterwordspacing}{\spaceskip=\fontdimen2\font plus
\BIBentryALTinterwordstretchfactor\fontdimen3\font minus \fontdimen4\font\relax}
\providecommand{\BIBforeignlanguage}[2]{{%
\expandafter\ifx\csname l@#1\endcsname\relax
\typeout{** WARNING: IEEEtran.bst: No hyphenation pattern has been}%
\typeout{** loaded for the language `#1'. Using the pattern for}%
\typeout{** the default language instead.}%
\else
\language=\csname l@#1\endcsname
\fi
#2}}
\providecommand{\BIBdecl}{\relax}
\BIBdecl

\bibitem{chen2019quantile}
X.~Chen, W.~Liu, and Y.~Zhang, ``Quantile regression under memory constraint,'' \emph{The Annals of Statistics}, vol.~47, no.~6, pp. 3244--3273, Dec. 2019.

\bibitem{zhou2022admm}
X.~Zhou and Y.~Xiang, ``\text{ADMM}-based differential privacy learning for penalized quantile regression on distributed functional data,'' \emph{Mathematics}, vol.~10, no.~16, p. 2954, Aug. 2022.

\bibitem{yin2021comprehensive}
X.~Yin, Y.~Zhu, and J.~Hu, ``A comprehensive survey of privacy-preserving federated learning: A taxonomy, review, and future directions,'' \emph{ACM Computing Surveys}, vol.~54, no.~6, pp. 1--36, July 2021.

\bibitem{zhou2023decentralized}
X.~Zhou, W.~Liang, I.~Kevin, K.~Wang, Z.~Yan, L.~T. Yang, W.~Wei, J.~Ma, and Q.~Jin, ``Decentralized p2p federated learning for privacy-preserving and resilient mobile robotic systems,'' \emph{IEEE Wireless Communications}, vol.~30, no.~2, pp. 82--89, Apr. 2023.

\bibitem{zhao2022participant}
J.~Zhao, X.~Chang, Y.~Feng, C.~H. Liu, and N.~Liu, ``Participant selection for federated learning with heterogeneous data in intelligent transport system,'' \emph{IEEE Transactions on Intelligent Transportation Systems}, Feb. 2022.

\bibitem{tsouvalas2022federated}
V.~Tsouvalas, A.~Saeed, T.~{\"O}z{\c{c}}elebi, and N.~Meratnia, ``Federated learning with noisy labels: Achieving generalization in the face of label noise,'' in \emph{First Workshop on Interpolation Regularizers and Beyond at NeurIPS 2022}, Dec. 2022.

\bibitem{koenker1987algorithm}
R.~W. Koenker and V.~d'Orey, ``Algorithm as 229: Computing regression quantiles,'' \emph{Applied Statistics}, pp. 383--393, Jan. 1987.

\bibitem{yu2020probabilistic}
Y.~Yu, X.~Han, M.~Yang, and J.~Yang, ``Probabilistic prediction of regional wind power based on spatiotemporal quantile regression,'' \emph{IEEE Transactions on Industry Applications}, vol.~56, no.~6, pp. 6117--6127, Dec. 2020.

\bibitem{taieb2016fore}
S.~Ben~Taieb, R.~Huser, R.~J. Hyndman, and M.~G. Genton, ``Forecasting uncertainty in electricity smart meter data by boosting additive quantile regression,'' \emph{IEEE Transactions on Smart Grid}, vol.~7, no.~5, pp. 2448--2455, Mar. 2016.

\bibitem{happy2021stat}
H.~Aprillia, H.-T. Yang, and C.-M. Huang, ``Statistical load forecasting using optimal quantile regression random forest and risk assessment index,'' \emph{IEEE Transactions on Smart Grid}, vol.~12, no.~2, pp. 1467--1480, Oct. 2021.

\bibitem{wu2009variable}
Y.~Wu and Y.~Liu, ``Variable selection in quantile regression,'' \emph{Statistica Sinica}, vol.~19, no.~2, pp. 801--817, Apr. 2009.

\bibitem{xue2012positive}
L.~Xue, S.~Ma, and H.~Zou, ``Positive-definite l1-penalized estimation of large covariance matrices,'' \emph{Journal of the American Statistical Association}, vol. 107, no. 500, pp. 1480--1491, Dec. 2012.

\bibitem{he2016regularized}
Q.~He, L.~Kong, Y.~Wang, S.~Wang, T.~A. Chan, and E.~Holland, ``Regularized quantile regression under heterogeneous sparsity with application to quantitative genetic traits,'' \emph{Computational Statistics \& Data Analysis}, vol.~95, pp. 222--239, Mar. 2016.

\bibitem{algamal2018gene}
Z.~Y. Algamal, R.~Alhamzawi, and H.~T.~M. Ali, ``Gene selection for microarray gene expression classification using bayesian lasso quantile regression,'' \emph{Computers in Biology and Medicine}, vol.~97, pp. 145--152, June 2018.

\bibitem{tibshirani2014adaptive}
R.~J. TIBSHIRANI, ``Adaptive piecewise polynomial estimation via trend filtering,'' \emph{The Annals of Statistics}, vol.~42, no.~1, pp. 285--323, Feb. 2014.

\bibitem{chen2021quantile}
I.-C. Chen, S.~J. Bertke, and B.~D. Curwin, ``Quantile regression for exposure data with repeated measures in the presence of non-detects,'' \emph{Journal of Exposure Science \& Environmental Epidemiology}, vol.~31, no.~6, pp. 1057--1066, Nov. 2021.

\bibitem{zhang2010nearly}
C.-H. Zhang, ``Nearly unbiased variable selection under minimax concave penalty,'' \emph{The Annals of statistics}, vol.~38, no.~2, pp. 894--942, Apr. 2010.

\bibitem{fan2001variable}
J.~Fan and R.~Li, ``Variable selection via nonconcave penalized likelihood and its oracle properties,'' \emph{Journal of the American statistical Association}, vol.~96, no. 456, pp. 1348--1360, Dec. 2001.

\bibitem{mirzaeifard2022dynamic}
R.~Mirzaeifard, V.~C. Gogineni, N.~K. Venkategowda, and S.~Werner, ``Dynamic graph topology learning with non-convex penalties,'' in \emph{30th European Signal Processing Conference}.\hskip 1em plus 0.5em minus 0.4em\relax IEEE, 2022, pp. 682--686.

\bibitem{mirzaeifard2022robust}
R.~Mirzaeifard, N.~K. Venkategowda, and S.~Werner, ``Robust phase retrieval with non-convex penalties,'' in \emph{IEEE Asilomar Conference on Signals, Systems, and Computers}, Oct. 2022, pp. 1291--1295.

\bibitem{mirzaeifard2022admm}
R.~Mirzaeifard, N.~K. Venkategowda, V.~C. Gogineni, and S.~Werner, ``{ADMM} for sparse-penalized quantile regression with non-convex penalties,'' in \emph{European Signal Processing Conference (EUSIPCO)}, Aug. 2022, pp. 2046--2050.

\bibitem{belloni2011l1}
A.~Belloni and V.~Chernozhukov, ``$l_1$-penalized quantile regression in high-dimensional sparse models,'' \emph{The Annals of Statistics}, vol.~39, no.~1, pp. 82--130, Feb. 2011.

\bibitem{koenker2005frisch}
R.~Koenker and P.~Ng, ``A frisch-newton algorithm for sparse quantile regression,'' \emph{Acta Mathematicae Applicatae Sinica}, vol.~21, no.~2, pp. 225--236, May 2005.

\bibitem{wang2017distributed}
H.~Wang and C.~Li, ``Distributed quantile regression over sensor networks,'' \emph{IEEE Transactions on Signal and Information Processing over Networks}, vol.~4, no.~2, pp. 338--348, Apr. 2017.

\bibitem{boyd2011distributed}
S.~Boyd, N.~Parikh, and E.~Chu, \emph{Distributed Optimization and Statistical Learning via the Alternating Direction Method of Multipliers}.\hskip 1em plus 0.5em minus 0.4em\relax Now Publishers Inc, 2011.

\bibitem{peng2015iterative}
B.~Peng and L.~Wang, ``An iterative coordinate descent algorithm for high-dimensional nonconvex penalized quantile regression,'' \emph{Journal of Computational and Graphical Statistics}, vol.~24, no.~3, pp. 676--694, July 2015.

\bibitem{zou2008one}
H.~Zou and R.~Li, ``One-step sparse estimates in nonconcave penalized likelihood models,'' \emph{Annals of Statistics}, vol.~36, no.~4, p. 1509, Aug. 2008.

\bibitem{davis2019stochastic}
D.~Davis and D.~Drusvyatskiy, ``Stochastic model-based minimization of weakly convex functions,'' \emph{SIAM Journal on Optimization}, vol.~29, no.~1, pp. 207--239, Jan. 2019.

\bibitem{chen2020asynchronous}
Y.~Chen, Y.~Ning, M.~Slawski, and H.~Rangwala, ``Asynchronous online federated learning for edge devices with non-iid data,'' in \emph{2020 IEEE International Conference on Big Data (Big Data)}.\hskip 1em plus 0.5em minus 0.4em\relax IEEE, 2020, pp. 15--24.

\bibitem{li2019convergence}
X.~Li, K.~Huang, W.~Yang, S.~Wang, and Z.~Zhang, ``On the convergence of fedavg on non-iid data,'' \emph{arXiv preprint arXiv:1907.02189}, 2019.

\bibitem{karimireddy2020scaffold}
S.~P. Karimireddy, S.~Kale, M.~Mohri, S.~Reddi, S.~Stich, and A.~T. Suresh, ``Scaffold: Stochastic controlled averaging for federated learning,'' in \emph{International conference on machine learning}.\hskip 1em plus 0.5em minus 0.4em\relax PMLR, Nov. 2020, pp. 5132--5143.

\bibitem{davis2022proximal}
D.~Davis and D.~Drusvyatskiy, ``Proximal methods avoid active strict saddles of weakly convex functions,'' \emph{Foundations of Computational Mathematics}, vol.~22, no.~2, pp. 561--606, 2022.

\bibitem{mirzaeifard2023distributed}
R.~Mirzaeifard, V.~C. Gogineni, N.~K.~D. Venkategowda, and S.~Werner, ``Distributed quantile regression with non-convex sparse penalties,'' in \emph{IEEE Statistical Signal Processing Workshop (SSP)}, July 2023, pp. 250--254.

\bibitem{mirzaeifard2023smoothing}
R.~Mirzaeifard, N.~K. Venkategowda, V.~C. Gogineni, and S.~Werner, ``Smoothing admm for sparse-penalized quantile regression with non-convex penalties,'' \emph{IEEE Open Journal of Signal Processing}, 2023.

\bibitem{gao2024stochastic}
W.~Gao and Q.~Deng, ``Stochastic weakly convex optimization beyond lipschitz continuity,'' \emph{arXiv preprint arXiv:2401.13971}, 2024.

\bibitem{koenker1982robust}
R.~Koenker and G.~Bassett~Jr, ``Robust tests for heteroscedasticity based on regression quantiles,'' \emph{Econometrica: Journal of the Econometric Society}, pp. 43--61, Jan. 1982.

\bibitem{varma2019vector}
R.~Varma, H.~Lee, J.~Kova{\v{c}}evi{\'c}, and Y.~Chi, ``Vector-valued graph trend filtering with non-convex penalties,'' \emph{IEEE Transactions on Signal and Information Processing over Networks}, vol.~6, pp. 48--62, Dec. 2019.

\bibitem{chen2012smoothing}
X.~Chen, ``Smoothing methods for nonsmooth, nonconvex minimization,'' \emph{Mathematical programming}, vol. 134, pp. 71--99, Aug. 2012.

\bibitem{bian2013neural}
W.~Bian and X.~Chen, ``Neural network for nonsmooth, nonconvex constrained minimization via smooth approximation,'' \emph{IEEE Transactions on Neural Networks and Learning Systems}, vol.~25, no.~3, pp. 545--556, Oct. 2013.

\bibitem{yu2022smoothing}
Q.~Yu and X.~Zhang, ``A smoothing proximal gradient algorithm for matrix rank minimization problem,'' \emph{Computational Optimization and Applications}, pp. 1--20, 2022.

\bibitem{attouch2013convergence}
H.~Attouch, J.~Bolte, and B.~F. Svaiter, ``Convergence of descent methods for semi-algebraic and tame problems: proximal algorithms, forward--backward splitting, and regularized gauss--seidel methods,'' \emph{Mathematical Programming}, vol. 137, no. 1-2, pp. 91--129, Feb. 2013.

\bibitem{lima2014sparsity}
M.~V. Lima, T.~N. Ferreira, W.~A. Martins, and P.~S. Diniz, ``Sparsity-aware data-selective adaptive filters,'' \emph{IEEE Transactions on Signal Processing}, vol.~62, no.~17, pp. 4557--4572, Sep. 2014.

\bibitem{misc_communities_and_crime_unnormalized_211}
M.~Redmond, ``{Communities and Crime Unnormalized},'' UCI Machine Learning Repository, 2011, {DOI}: https://doi.org/10.24432/C5PC8X.

\end{thebibliography}
\end{document}